%% file: main.tex
\documentclass{article}

\input{settings}

\begin{document}

\title{Upper bounds on the number of causal pieces in spiking neural networks}

\author{
  Dominik Dold$^*$, Dennis Elbrächter$^*$, Philipp Petersen\\[2pt]
  \textit{\small{Faculty of Mathematics and Research Network Data Science @ Uni Vienna, University of Vienna}}\\[2pt]
  $^*$ \textit{\small{These authors contributed equally to this work.}}
}

\maketitle
\begin{abstract}
\input{content/abstract}
\end{abstract}

\section{Introduction}

\input{content/introduction}

\section{Related work}

\input{content/related}

\section{Background}

\input{content/methods}

\section{Combinatoric upper bounds for general spiking neuron models}\label{sec:combinatoric}

\input{content/proof_general}

\section{Derivation of hyperplane-based upper bound}
\input{content/hyperplanes}

\section*{Acknowledgments}

\input{content/acknowledgments}

\section*{Use of AI tools}

ChatGPT (OpenAI) was used iteratively for brainstorming, refinement of ideas, assistance with mathematical formalisation, and presentation. Parts of the manuscript were developed with its assistance and subsequently revised and verified by the authors. The authors take full responsibility for all mathematical statements, proofs, conclusions, and the final manuscript.

\printbibliography
\addcontentsline{toc}{section}{References}

\appendix
\section{Appendix}

\input{content/appendix}

\end{document}

%% file: settings.tex
\usepackage[activate={true,nocompatibility},final,tracking=true,kerning=true,spacing=true,factor=1100,stretch=10,shrink=10]{microtype}

\usepackage{arxiv}

\usepackage[font=small,labelfont={bf}, figurename=Figure, tablename=Table]{caption}

\usepackage[hidelinks]{hyperref}
\hypersetup{
  colorlinks   = true, 
  urlcolor     = esalightblue, 
  linkcolor    = esablue, 
  citecolor   = esablue 
}
\usepackage{amsfonts}       
\usepackage{amsmath}
\usepackage{amssymb}
\usepackage{amsthm}
\usepackage[utf8]{inputenc} 
\usepackage[T1]{fontenc}    
\usepackage{hyperref}       
\usepackage{url}            
\usepackage{booktabs}       
\usepackage{nicefrac}       
\usepackage{lipsum}
\usepackage{graphicx}
\usepackage{enumerate}
\graphicspath{ {./images/} }
\usepackage[nameinlink, capitalise]{cleveref}
\crefname{app}{appendix}{appendices}
\Crefname{app}{Appendix}{Appendices}
\usepackage{xcolor}
\usepackage{booktabs}
\definecolor{esablue}{RGB}{0,57,158}
\definecolor{esalightblue}{RGB}{0,155,219}
\definecolor{esared}{RGB}{150,1,54}
\usepackage{float}
\newcommand{\stirlingii}{\genfrac{\{}{\}}{0pt}{}}

\theoremstyle{plain}
\usepackage{amsthm}
\usepackage{aliascnt}

\newtheorem{theorem}{Theorem}[section]

\newaliascnt{proposition}{theorem}
\newtheorem{proposition}[proposition]{Proposition}
\aliascntresetthe{proposition}

\newaliascnt{lemma}{theorem}
\newtheorem{lemma}[lemma]{Lemma}
\aliascntresetthe{lemma}

\newaliascnt{corollary}{theorem}
\newtheorem{corollary}[corollary]{Corollary}
\aliascntresetthe{corollary}

\newaliascnt{assumption}{theorem}
\newtheorem{assumption}[assumption]{Assumption}
\aliascntresetthe{assumption}

\theoremstyle{definition}

\newaliascnt{definition}{theorem}
\newtheorem{definition}[definition]{Definition}
\aliascntresetthe{definition}

\theoremstyle{remark}

\newaliascnt{remark}{theorem}
\newtheorem{remark}[remark]{Remark}
\aliascntresetthe{remark}

\crefname{proposition}{Proposition}{Propositions}
\crefname{theorem}{Theorem}{Theorems}
\crefname{lemma}{Lemma}{Lemmas}
\crefname{corollary}{Corollary}{Corollaries}
\crefname{definition}{Definition}{Definitions}
\crefname{assumption}{Assumption}{Assumptions}
\crefname{remark}{Remark}{Remarks}

\usepackage{todonotes}

\newcommand{\N}{\mathbb{N}}
\newcommand{\R}{\mathbb{R}}
\newcommand{\cP}{\mathcal{P}}
\newcommand{\cC}{\mathcal{C}}

\usepackage[giveninits=true,sorting=none,citestyle=numeric-comp]{biblatex}
\usepackage{tcolorbox} 
\tcbuselibrary{breakable}
\definecolor{highlightgray}{RGB}{240, 240, 240} 
\tcbset{
  highlightbox/.style={
  colback=highlightgray,
  colframe=black!30,
  coltitle=black,
  boxrule=0pt,
  arc=4pt,
  left=0pt,
  right=0pt,
  top=0pt,
  bottom=0pt,
  fonttitle=\bfseries
}}

\newtcolorbox{highlightbox}{
breakable,
  colback=gray!10,
  colframe=gray!25,
  left=1pt,
  right=1pt,
  top=1pt,
  bottom=1pt,
  boxrule=0.pt,
  arc=4pt,
}

\newcommand{\causal}{\mathcal{C}}
\newcommand{\paths}{\mathcal{P}}

%% file: content/abstract.tex
The input space of single-spike neural networks decomposes into causal pieces, and the number of causal pieces has been shown to provide a lower bound on the approximation error for some spiking neural networks.
While this decomposition is applicable to a wide range of spiking neuron models, the limits of how many causal pieces can be realised by spiking neural networks are still unknown.
In this work, we derive two complementary upper bounds on the number of causal pieces in feedforward single-spike neural networks.
First, we establish a combinatorial bound given by the poly-Bernoulli numbers, which applies to a broad class of spiking neuron models.
The bound is of order $e^{2(L-1)N\log N+\mathcal{O}(N)}$, with network depth $L$ and maximum layer width $N$. 
Second, for a specific set of neuron models, we derive a bound based on hyperplane arrangements.
Applying the latter framework, we obtain bounds of order $\mathcal{O}\!\left(N^{(4L-3)(N_0-1)}\right)$ for the non-leaky integrate-and-fire model with exponential synapses and
the ReLU-type spike response model, and $\mathcal{O}\!\left(N^{L(N_0-1)}\right)$ for the leaky integrate-and-fire model with Dirac-$\delta$ synapses, where $N_0$ is the input dimension.
For the non-leaky integrate-and-fire model with Dirac-$\delta$ synapses, additional temporal structure allows us to sharpen the general combinatorial bound to one that depends only on $N_0$, independently of the network depth and widths.

%% file: content/introduction.tex
Recently, the concept of causal pieces (or causal regions) has been introduced to study the expressivity of spiking neural networks (SNNs) using single-spike coding \cite{dold2025causal} -- neural networks in which every neuron spikes at most once per inference \cite{mostafa2017supervised,goltz2021fast,comsa2020temporal}.
A causal piece is a region of the input space on which the same subnetwork is active.
Intuitively, this is the collection of neurons and synaptic connections through which the input spikes influence the output spike times (\cref{fig:intro}A,B).
In particular, the number of causal pieces has been shown to provide a lower bound on the approximation error \cite[Theorem~1]{dold2025causal},
\begin{equation}
    \|\Phi-g\|_{L^\infty(I)}
    >
    c\,\mathfrak{C}(\Phi)^{-2},
\end{equation}
for $g \in C^3$ such that $z\mapsto \exp(g(\tau\log z)/\tau)$ is not affine, where $\tau>0$ denotes the common synaptic time constant of the nLIF network.
$I\subset\mathbb{R}$ is a non-empty compact interval,
$\mathfrak{C}(\Phi)$ denotes the number of causal pieces of the SNN $\Phi$, and $c>0$ is a constant that depends on $g$, $\tau$, and $\max_{x \in I}\bigl\{\exp(g(x)/\tau), \exp(\Phi(x)/\tau)\bigr\}$.
While this result was originally stated for networks of non-leaky integrate-and-fire (nLIF) neurons with exponential synapses, the underlying proof mechanism is more general: it applies whenever, possibly after a suitable change of variables, the network realisation is affine on each causal piece.
This also holds for the ReLU-type response model
(\cref{def:relu_response}) considered here, whose spike-time map is affine on each causal piece by default.
Thus, analogous approximation lower bounds follow for this model.
Furthermore, approximation bounds of this type have also been extended to multi-spike neural networks using a mathematical transference principle \cite{dold2026equivalence}.

Consequently, a low number of causal pieces limits the approximation capabilities of SNNs, and any upper bound on the number of causal pieces $\mathfrak{C}(\Phi)\leq U$ immediately yields
\begin{equation}
    \|\Phi - g \|_{L^\infty(I)}
    >
    c \cdot U^{-2}.
\end{equation}
But how many causal pieces can SNNs have?
In this work, we provide combinatoric upper bounds that are generally applicable to a large set of SNN models.
We further derive upper bounds using hyperplane arrangement arguments that only apply to specific SNN models.
\begin{figure}[t!]
    \centering
    \includegraphics[width=\linewidth]{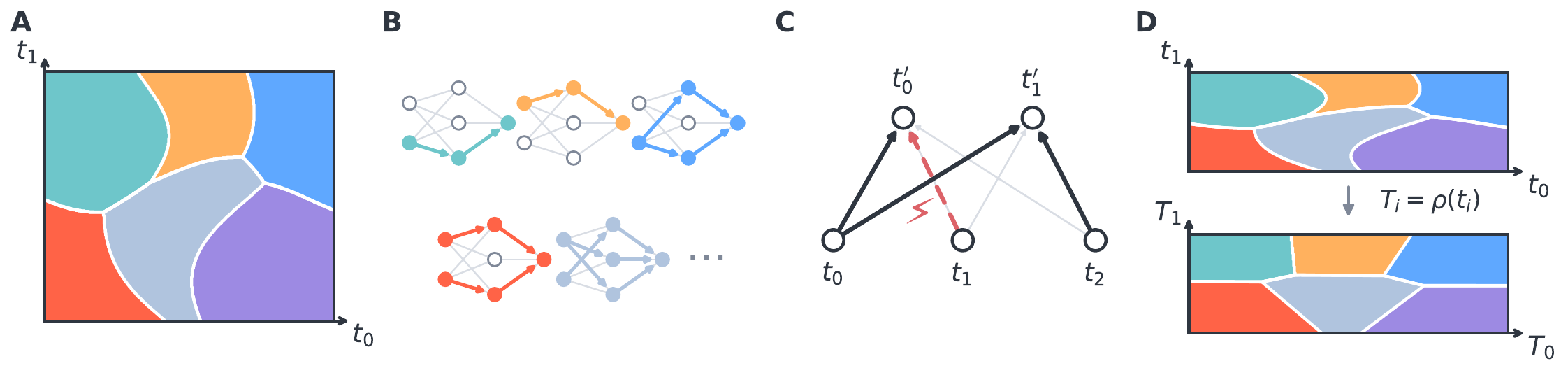}
    \caption{\textbf{A.} The input space (here: two-dimensional) of a SNN partitions into separate regions with, potentially, curved boundaries. Each region (causal pieces) is denoted by a different colour. 
    \textbf{B.} Every causal piece corresponds to a different subnetwork (coloured) that produces, in that input region, the output. Corresponding regions in (A) are coded here by colour. Thus, counting causal pieces corresponds to counting possible subnetworks.
    \textbf{C.} Not all synaptic activity patterns between two layers in a SNN are allowed. We show here one such forbidden pattern: $t'_0$ is caused by input neurons $\{0, 1\} = C_0$, and $t'_1$ by neurons $\{0, 2\}= C_1$. However, if $t'_1 \geq t'_0$, then $C_0$ has to be a subset of $C_1$.
    Similarly, if $t'_1 < t'_0$, $C_1$ has to be a subset of $C_0$.
    Neither is the case, and thus the pattern is not allowed.
    \textbf{D.} For specific neuron models, there exists a change of variables through which causal pieces become polyhedra.}
    \label{fig:intro}
\end{figure}

The first two bounds we introduce are combinatorial and independent of the particular neuron dynamics.
While the first bound is generally applicable to feedforward single-spike neural networks, the second one only applies to SNNs in which the presynaptic spike times are the same for all neurons within a layer.
This includes the standard setting with zero synaptic delays as well as settings where delays are only on the axons, i.e., on the side of the presynaptic neurons.
Most importantly, these upper bounds do not require synaptic weights to be exclusively positive, and they are not restricted to particular neuron models.

The two combinatorial bounds are derived consecutively.
First, given an input, a causal piece is defined by the active subnetwork, i.e., the set of neurons and connections that are involved in determining the output spike times (\cref{fig:intro}B).
Consequently, a naive upper bound on the number of causal pieces is given by the number of possible subsets of all synapses (\cref{prop:trivialBound}):
\begin{highlightbox}
\begin{equation}
\mathfrak{C}_{\diamond}(\Phi)
\leq
\prod_{\ell=1}^{L}
2^{N_{\ell-1} N_\ell} = 2^{\sum_{\ell=1}^{L} N_{\ell-1} N_\ell}\,,
\end{equation}
\end{highlightbox}
where $\mathfrak{C}_\diamond(\Phi)$ is the number of causal pieces of spiking neural network $\Phi$ with $L \in \mathbb{N}$ layers, widths $N_1, ..., N_L \in \mathbb{N}$, and input dimension $N_0 \in \mathbb{N}$.
$\diamond\in\{<,\leq\}$ distinguishes whether neurons in the causal set must spike strictly before ($<$), or may also spike simultaneously with ($\leq$), the output neuron (\cref{def:cpiece}).

However, spiking neurons within a layer have strong temporal constraints: for any fixed input, if postsynaptic neurons $A$ and $B$ have finite spike times and $A$ spikes before $B$, then the causal set of $A$ is a subset of the causal set of $B$ (\cref{fig:intro}C) -- where the causal set contains the indices of the presynaptic neurons whose spikes occur before the postsynaptic spike (see \cref{def:cset}).
Thus, ordering the neurons by their finite spike times produces an increasing sequence of causal sets. 
Neurons that do not spike have empty causal sets.
This nesting phenomenon has also been identified concurrently in \cite{singh2026polyhedral} and used to derive upper bounds for a restricted class of SNNs with shared positive weights.
Here, we show that, for arbitrary synaptic weights, this temporal constraint is equivalent to requiring the matrix of active connections in each layer to be a lonesum matrix.
Using exact counting results for lonesum matrices \cite{brewbaker2008combinatorial}, we obtain an upper bound given by a product of poly-Bernoulli numbers, which is substantially tighter than the first upper bound we derived (\cref{thm:lonesum_bound}):
\begin{highlightbox}
\begin{equation}
    \mathfrak{C}_{\diamond}(\Phi)
\leq
\prod_{\ell=1}^{L}
\mathcal{B}_{N_\ell}^{-N_{\ell-1}} \,,
\end{equation}
\end{highlightbox}
\noindent for $\diamond\in\{<,\leq\}$. 
The poly-Bernoulli numbers are given by \cite{kaneko1997poly,arakawa1999poly,brewbaker2008combinatorial}
\begin{equation}\label{eq:polybernoulli}
\mathcal{B}^{(-k)}_n = \sum_{m=0}^{\min(n,k)} m! \stirlingii{n+1}{m+1} m!\stirlingii{k+1}{m+1}     \,,
\end{equation}
with the Stirling number of the second kind \cite{rennie1969stirling}
\begin{equation}\label{eq:stirlingii}
\stirlingii{n}{k} = \frac{(-1)^k}{k!} \sum_{\ell=0}^k (-1)^\ell \binom{k}{\ell} \ell^n \,,
\end{equation}
which is the number of partitions of $n$ elements into $k$ non-empty sets.

Lastly, we derive upper bounds for a specific class of spiking neuron models.
For these models, a reparameterisation of spike times transforms causal pieces into polyhedral regions on which the SNN realises an affine function (\cref{fig:intro}D).
The exact requirements are stated in \cref{ass:model}.
In addition, we assume no synaptic delays.
The covered models include the nLIF neuron model with exponential synapses studied in \cite{dold2025causal}, the ReLU-type spike response model used in \cite{singh2026polyhedral}, the nLIF model with Dirac-$\delta$ synapses \cite{kheradpisheh2020temporal}, and the leaky integrate-and-fire (LIF) model with Dirac-$\delta$ synapses -- again with no constraints on the sign of synaptic weights.
For such models, the number of causal pieces can be upper bounded by estimating the number of regions formed by the corresponding hyperplane arrangements, similarly to region-counting arguments for ReLU-based artificial neural networks (ANNs) \cite{montufar2014number,montufar2017notes}.
Our bound additionally tracks the dimension of the image of each intermediate layer, which can yield tighter bounds than our combinatoric estimates for networks containing low-dimensional bottleneck layers (\cref{thm:hyper_general}).
Specifically, for the nLIF model with exponential synapses and the ReLU-type response model, we find (\cref{thm:nLIF_hyperplane,thm:ReLU_hyperplane}):
\begin{highlightbox}
\begin{equation}
\textbf{nLIF+Exp. / ReLU-type:}\ \ \mathfrak{C_{<}}(\Phi) \leq \prod_{\ell = 1}^L \sum_{j=0}^{d_{\ell-1}} 2^j{{N_{\ell-1}\choose 2}\choose j}\min\left\{N_{\ell-1}^{N_\ell},\ \sum_{j=0}^{d_{\ell-1}}2^j{N_\ell (N_{\ell-1}-1)\choose j}\right\},
\end{equation}
\end{highlightbox}
with $d_\ell:=\min_{0\leq r\leq\ell} (N_r-1)$, $\ell\in\{0,\dots,L\}$. 
For the nLIF/LIF model with Dirac-$\delta$ (or instantaneous) synapses, we have (\cref{thm:LIF,thm:nLIF_Dirac}):
\begin{highlightbox}
\begin{equation}
        \textbf{LIF:}\ \ \mathfrak{C}_{\diamond}(\Phi) \leq R(N_0) \cdot \prod_{\ell=1}^L \min\left\{(d_{\ell-1}+1)^{N_\ell},\ \sum_{j=0}^{d_{\ell-1}}2^j{N_\ell d_{\ell-1}\choose j}\right\}, \ \textbf{nLIF:}\ \ \mathfrak{C}_{\diamond}(\Phi) \leq R(N_0) \,,
\end{equation}
\end{highlightbox}
for $\diamond \in \{<, \leq \}$ (see \cref{def:cpiece}) and with $R(N_0) = \sum_{q\in[N_0]}q!{N_0\brace q}$.
The two combinatoric bounds and the hyperplane bounds are compared for different network architectures in \cref{fig:bounds}.
For completeness, we also show the general bound derived in \cite[Theorem 7.1]{singh2026polyhedral} for these architectures.
\begin{figure}[hb!]
    \centering
    \includegraphics[width=\linewidth]{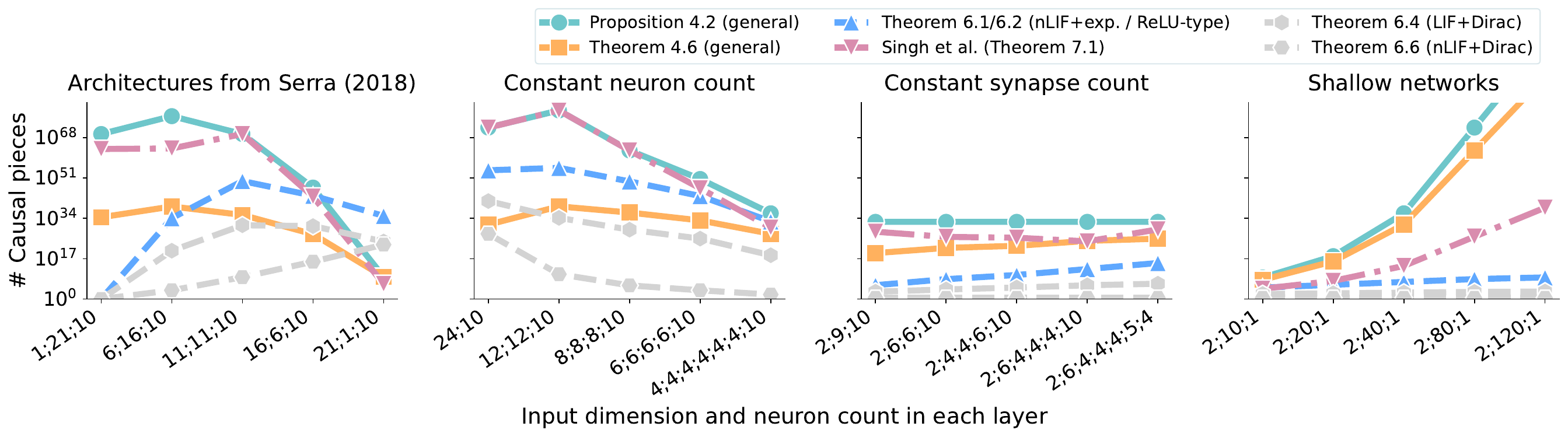}
    \caption{Upper bounds as stated in \cref{prop:trivialBound} and \cref{thm:lonesum_bound} for general SNNs, and for specific SNNs using hyperplane arrangement arguments (\cref{thm:nLIF_hyperplane,thm:ReLU_hyperplane,thm:LIF,thm:nLIF_Dirac}), evaluated for different neural architectures. We denote here ``input dimension; hidden dimension; hidden dimension; $\dots$; output dimension''. 
    As some of our bounds also apply to the neuron model studied in \cite{singh2026polyhedral}, we include here the upper bound stated in \cite[Theorem 7.1]{singh2026polyhedral}.
    The first set of architectures is taken from \cite{serra2018bounding}.}
    \label{fig:bounds}
\end{figure}

The hyperplane approach yields bounds of the form $N^{\mathcal{O}(N_0 L)}$ for
the models considered here, where $N$ is the maximum width and $N_0$ the input dimension. This is due to  
\begin{align*}
  \max_{0\leq\ell \leq L}d_{\ell}=d_0=(N_0-1). 
\end{align*}
While $N_0$ enters the exponent, the dependence on the maximum width remains polynomial. This is of particular interest as the input dimension $N_0$ is generally given by the task at hand rather than chosen as part of the network architecture. 
\cref{tab:width-asymptotics} provides the asymptotic scaling of the bounds with respect to maximum width $N$, given fixed depth and input dimension.    
The bound for nLIF networks with Dirac-$\delta$ synapses depends only on the input dimension, independently of the depth and the widths of all subsequent layers.
\begin{table}[t!]
    \centering
    \begin{tabular}{@{}p{0.58\linewidth}p{0.36\linewidth}@{}}
        \toprule
        Model
        & Asymptotic upper bound for causal pieces\\
        \midrule
        
        nLIF with exponential synapses (\Cref{thm:nLIF_hyperplane})
        & $\mathcal{O}\!\left(N^{(4L-3)(N_0-1)}\right)$ \\[0.6em]
        
        ReLU response (\Cref{thm:ReLU_hyperplane})
        & $\mathcal{O}\!\left(N^{(4L-3)(N_0-1)}\right)$ \\[0.6em]

        LIF with Dirac-$\delta$ synapses (\Cref{thm:LIF})
        & $\mathcal{O}\!\left(N^{L(N_0-1)}\right)$ \\[0.6em]

        nLIF with Dirac-$\delta$ synapses (\Cref{thm:nLIF_Dirac})
        & $\mathcal{O}(1)$ \\[0.6em]

        Combinatorial bound (\Cref{thm:lonesum_bound})
        & $e^{2(L-1)N\log N+\mathcal{O}(N)}$ \\
        \bottomrule
    \end{tabular}\\[5pt]
    \caption{Asymptotic behaviour of the upper bounds w.r.t.\@ maximum width $N$ for fixed input dimension $N_0$ and depth $L$.
    }
    \label{tab:width-asymptotics}
\end{table}
For the ReLU-response model, \cite[Proposition~6.3]{singh2026polyhedral} show that, for fixed input dimension $N_0$, positive weights, and per-neuron threshold, there exist shallow networks of width $N$ with $\mathfrak C(\Phi) = \Omega\!\left(N^{N_0-1}\right)$.
For $L=1$, this matches our upper bound $\mathcal O(N^{N_0-1})$ in its dependence on $N$.

The exponential dependence on the bottleneck dimension and depth is unavoidable already for ReLU-response networks. 
Singh et al.~\cite[Lemma~7.1]{singh2026polyhedral}, with $m=3$, constructs a two-layer network of widths $(2,3,2)$ with two distinct causal patterns whose input domains map bijectively onto the same output domain. Iterating this construction gives $2^K$ patterns after $K$ blocks, as established in \cite[Theorem~7.2]{singh2026polyhedral}.\footnote{The construction in \cite{singh2026polyhedral} establishes a rigorous theoretical lower bound, but the contraction of relative spike-time intervals after each two-layer block may limit its practical realisation in the presence of finite temporal resolution or analogue noise.} 
Following the parallelisation discussed immediately thereafter, and allowing zero synaptic weights, one can run $r$ copies on disjoint input pairs by setting all cross-pair weights to zero.
Their neuron-dependent thresholds can be replaced by a common threshold without changing any spike times, by rescaling the incoming weights of each neuron accordingly.

For $r,K\in\mathbb N$, this yields a network $\Phi$ of architecture $(2r,3r,2r,\ldots,3r,2r)$, with $K$ blocks. The distinct patterns remain distinguishable under our causal subnetwork definition, as, within each two-input block, the second-layer neurons include all first-layer neurons as causal predecessors, and one of those first-layer neurons includes both block inputs. Consequently, with depth $L=2K$ and bottleneck dimension $d=2r-1$,
\begin{align}
    C_<(\Phi)\ge 2^{rK}=2^{(d+1)L/4}.
\end{align}
Thus, even with all widths proportional to $d$, a general upper bound must allow exponential growth in $dL$. 
For each fixed architecture, the lower bound also extends to nLIF neural networks with exponential synapses in a parameter regime where their response sufficiently approximates the ReLU-type response (see \Cref{app:nlif-lower-bound}). 
This argument does not establish the lower bound for nLIF neural networks outside that regime.
Moreover, we only show that exponential dependence on $dL$ is unavoidable, but do not establish the optimality of the exponents in our upper bounds. Closing the remaining gap between the lower and upper bounds remains future work.
In the case of Dirac-$\delta$ synapses, things are a bit simpler and we get an exactly matching lower bound for nLIF and a matching lower bound, up to a constant, for LIF (see \Cref{sec:Dirac_lower}).

%% file: content/related.tex
Worst-case upper bounds on the number of causal pieces have previously been derived for individual spiking neurons and for deep SNNs \cite{singh2023expressivity,dold2025causal}.
Compared to these initial results, the upper bounds we derive in this work are much tighter.

Concurrently to this work, Singh et al.~\cite{singh2026polyhedral} derived upper and lower bounds on the number of causal pieces (or causal regions) of positive-weight time-to-first-spike networks with a ReLU-type postsynaptic response (see also \cite{singh2023expressivity,neuman2024stable} for details on the used neuron model).
They observe that the causal sets of neurons form a nested sequence and derive an exact count for the maximum number of realisable causal set tuples for shallow neural networks where all neurons share the same positive weight vector (see \cite[Example~6.2]{singh2026polyhedral} and \cite[Proposition~6.4]{singh2026polyhedral}).
In contrast, our combinatorial result does not require shared weights or positive synaptic weights.
For any fixed input, ordering the postsynaptic neurons by their finite spike times produces a nested sequence of causal sets.
We show that, when viewed in the original neuron labelling, this condition is precisely the lonesum matrix condition, allowing us to use exact enumeration results from the combinatorics literature.

In addition, Singh et al.~\cite{singh2026polyhedral} show that causal pieces are polyhedral regions for their positive-weight ReLU-response time-to-first-spike model and derive upper bounds using hyperplane arrangements.
We use a related approach for a larger class of spiking neuron models, and count the relevant hyperplanes only after conditioning on the presynaptic spike order.
Although \cite[Remark~7.1]{singh2026polyhedral} notes that intermediate affine images may be lower-dimensional, their general deep-network bound (Theorem 7.1) uses the full presynaptic dimension. Our bound explicitly propagates the reachable dimension from layer to layer, which can substantially reduce the region count in networks containing bottleneck layers.

Piece counting has been used extensively to quantify the expressive power of ANNs with ReLU activation function, or generally piecewise-affine activation functions.
In these cases, the approximation error is lower bounded by an expression that also depends on the inverse square of the number of affine regions the ANN expresses \cite{frenzen2010number}.
How many affine regions (or pieces) an ANN can have, has been upper bounded using hyperplane arguments \cite{montufar2014number,montufar2017notes,raghu2017expressive,serra2018bounding,hinz2019framework,goujon2024number}, with ReLU ANNs having bounds of order $\mathcal{O}(N^{LN_0})$ for fixed input dimension $N_0$ and depth $L$.
Our hyperplane-based bounds for SNNs take into account bottleneck dimensions of intermediate layers (see $d_\ell$ in \cref{thm:hyper_general}).
This has also been done for refining bounds of the number of affine pieces, derived using hyperplane arrangement arguments, in \cite[Proposition 3]{montufar2017notes}, and further used in \cite[Theorem 1, see also Lemma 5]{serra2018bounding}.
Moreover, Hanin \& Rolnick~\cite{hanin2019complexity,hanin2019deep} highlighted that the number of affine pieces observed in practice is much lower than expected from these bounds, showing the importance of also bounding or estimating the average number of affine pieces in ANNs (see, e.g., \cite{goujon2024number}).

Alternative approaches to quantifying the expressivity of SNNs have also been proposed recently.
Hundrieser et al.~\cite{hundrieser2025universal} study feedforward SNNs as maps from spike trains to spike trains and establish a universal representation theorem.
They further show near-optimality of the required network size by deriving nearly matching lower bounds.
As part of this analysis, they quantify expressivity by the number of distinct functions that SNNs with a fixed architecture can realise on a finite set of input spike trains.

Finally, Neuman et al.~\cite{neuman2024stable} study SNNs composed of ReLU-type response neurons with positive synaptic weights and equipped with an affine encoder and decoder.
For this architecture, they derive generalisation bounds using covering numbers and establish rate-optimal approximation of smooth functions as well as dimension-independent approximation of Barron regular functions.

%% file: content/methods.tex
For $q\in\mathbb{N}$, we use the notation $[q]:=\{1,\dots,q\}$, $[0] := \emptyset$, and for $J\subseteq\N$, we write $D^J:=\{x=(x_j)_{j\in J}:x_j\in D\}$. Given a single-spike network with depth $L\in\N$ and layer widths $N_0,\dots,N_L\in\N$, we denote by $t^{(\ell)}_i(b)$ the spike time of the $i$-th neuron in the $\ell$-th layer for input $b\in\R^{N_0}$. We set $t^{(0)}(b):=b$.

\textit{Causal sets} are defined for individual neurons.
Given an input, the causal set contains the indices of all presynaptic neurons that caused the neuron to fire (\cref{def:cset}).
\begin{highlightbox}
\begin{definition}[Causal set]\label{def:cset}
    Let $L\in\N$, $\ell\in [L]$, $i\in [N_\ell]$, and $b\in\R^{N_0}$.
    The \textbf{exclusive causal set} of the $i$-th neuron in the $\ell$-th layer is given by 
    \begin{align*}
        \causal^{(\ell)}_{i,<}(b):=\begin{cases}
            \{j\in[N_{\ell-1}]: t^{(\ell-1)}_j(b) < t^{(\ell)}_i(b) \}, & t^{(\ell)}_i(b) <\infty\\
            \emptyset, &t^{(\ell)}_i(b)=\infty
        \end{cases}.
    \end{align*}
    The \textbf{inclusive causal set} of the $i$-th neuron in the $\ell$-th layer is given by 
    \begin{align*}
        \causal^{(\ell)}_{i,\leq}(b):=\begin{cases}
            \{j\in[N_{\ell-1}]: t^{(\ell-1)}_j(b) \leq t^{(\ell)}_i(b) \}, & t^{(\ell)}_i(b) <\infty\\
            \emptyset, &t^{(\ell)}_i(b)=\infty
        \end{cases}.
    \end{align*}
\end{definition}
\end{highlightbox}
We use two slightly different notions of causal sets, since we consider different neuron models. 
Conceptually, the exclusive version makes sense for, e.g.\@ nLIF with exponential synapses or ReLU-type response, as in these cases the contribution of an incoming spike to the potential starts at $0$. 
Thus, an incoming spike at exactly the time of the outgoing spike does not have a causal impact, and including it would allow for degenerate cases where (transformed) affinity does not hold on the causal pieces.   
In contrast, for neurons with Dirac-$\delta$ synapses, an incoming spike contributes immediately.
Thus, an inclusive definition is more appropriate in this case. 

The notion of a causal subnetwork extends causal sets to deep networks (\cref{def:cpath}).
\begin{highlightbox}
\begin{definition}[Causal subnetwork] \label{def:cpath}
Let $L \in \mathbb{N}$, $\ell \in [L]$, $N_0,\dots,N_\ell \in \mathbb{N}$, $I \subseteq [N_\ell]$, $b\in\R^{N_0}$, and $\diamond\in\{<,\leq\}$.
We define
\begin{align*}
    \mathcal{P}^{(\ell)}_{I,\ell,\diamond}(b):=\{\big(i,\mathcal{C}^{(\ell)}_{i,\diamond}(b)\big)\colon i\in I\},
\end{align*}
and recursively, for $n=\ell,\dots,2$,
\begin{align*}
    \mathcal{P}^{(\ell)}_{I,n-1,\diamond}(b):=\bigcup_{(i,\mathcal{C})\in\mathcal{P}^{(\ell)}_{I,n,\diamond}(b)}\{\big(j,\mathcal{C}^{(n-1)}_{j,\diamond}(b)\big)\colon j\in\mathcal{C} \}.
\end{align*}
The \textbf{causal subnetwork} for layer $\ell$, index set $I$ and input $b$ is
\begin{align*}
    \mathcal{P}^{(\ell)}_{I,\diamond}(b):=\{\big(n,\mathcal{P}^{(\ell)}_{I,n,\diamond}(b)\big)\colon n\in[\ell]\}.
\end{align*}
\end{definition} 
\end{highlightbox}
\vspace{-4pt}
The causal subnetwork is the set of active synapses that were involved in calculating the output of a deep SNN given an input.
Specifically, each pair $(i,\mathcal{C})\in\mathcal{P}^{(\ell)}_{I,n,\diamond}(b)$ records the index of a neuron in layer $n$ together with its causal set. Starting from the neurons indexed by $I$, the recursion includes precisely those neurons of the preceding layer whose indices occur in these causal sets. 
The pairs $(n,\mathcal P_{I,n,\diamond}^{(\ell)}(b))$ retain the layer to which each collection belongs. 

Following the definition of causal sets and subnetworks, we can define a \textit{causal piece} (\cref{def:cpiece}).
\begin{highlightbox}
\begin{definition}[Causal piece / region]\label{def:cpiece} Let $L \in \mathbb{N}$, $\ell \in [L]$, $N_0,\dots,N_\ell \in \mathbb{N}$, $I \subseteq [N_\ell]$, $a\in\R^{N_0}$, and $\diamond\in\{<,\leq\}$. Then the \textbf{causal piece} associated with the causal subnetwork $\paths^{(\ell)}_{I,\diamond}(a)$ is
\begin{align*}
   \mathbb{P}[\paths^{(\ell)}_{I,\diamond}(a)]:=\{b\in\R^{N_0}\colon \paths^{(\ell)}_{I,\diamond}(b)=\paths^{(\ell)}_{I,\diamond}(a)\}. 
\end{align*}

\label{def:pieces}
\end{definition}
\end{highlightbox}
Intuitively, a causal piece is a region in the input space for which the output of the SNN is always produced by the same causal subnetwork.
While originally in \cite{dold2025causal}, causal pieces are defined for the input and parameter space, we here restrict it to inputs only, assuming SNNs with constant weights.
In this work, we are primarily interested in the number of causal pieces (\cref{def:numPieces}).
\begin{highlightbox}
\begin{definition}[Number of causal pieces]\label{def:numPieces}
Let $\Phi$ be a feedforward single-spike network with depth
$L\in\mathbb N$ and layer widths $N_0,\ldots,N_L\in\mathbb N$. For $\diamond\in\{<,\leq\}$, the number of causal pieces of $\Phi$ is
\begin{align}
    \mathfrak{C}_{\diamond}(\Phi):=
    \left|
    \left\{
        \mathcal{P}^{(L)}_{[N_L],\diamond}(b)
        : b\in\mathbb R^{N_0}
    \right\}
    \right|.
\end{align}
\end{definition}
\end{highlightbox}

We study four neuron models in particular here: the ReLU-type response model \cref{def:relu_response}, nLIF with exponential or Dirac-$\delta$ synapses (\cref{def:nLIF,def:dirac}), and LIF with Dirac-$\delta$ synapses (\cref{def:instant}).
Throughout the paper we use the convention $\inf\emptyset:=\infty$. In particular, this means that for neurons whose potential never reaches the threshold we consider the value of their spike-time to be $\infty$.
\begin{highlightbox}
\begin{definition}[ReLU-type response model]\label{def:relu_response}
    Let $N_0 \in \mathbb{N}$, $\vartheta  > 0$ be the spike threshold, $w \in \mathbb{R}^{N_0}$ the synaptic weights, and $s \in (\mathbb{R}\cup\{\infty\})^{N_0}$ the input spike times. 
    Then the membrane potential $u(t) \in \mathbb{R}$ of a \textbf{ReLU-type response neuron} at time $t \in \mathbb{R}$ is given by:
\begin{equation}    
    u(t) = \sum_{i\in[N_0], s_i<\infty} w_i \max(0, t-s_i) \,.
\end{equation}
Its spike time is $t^* = \mathrm{inf}\{t: u(t) = \vartheta\}$.
If $t^*<\infty$, its exclusive causal set $\mathcal{C}$ satisfies $\sum_{i\in\mathcal{C}}w_i>0$ and
\begin{equation}
    t^*
    =
    \frac{\vartheta + \sum_{i\in\mathcal{C}} w_i s_i}
         {\sum_{i\in\mathcal{C}} w_i} \,.
\end{equation}
 
\end{definition}
\end{highlightbox}
\begin{highlightbox}
\begin{definition}[nLIF with exponential synapses]\label{def:nLIF}
    Let $N_0 \in \mathbb{N}$, $\tau \in \mathbb{R}^+$ be the time constant, $\vartheta > 0$ the spike threshold, $w \in \mathbb{R}^{N_0}$ the synaptic weights, and $s \in (\mathbb{R}\cup\{\infty\})^{N_0}$ the input spike times. 
    Then the membrane potential $u(t) \in \mathbb{R}$ of a \textbf{nLIF neuron with exponential synapses} at time $t \in \mathbb{R}$ is given by:
\begin{equation}
    u(t)
    :=
    \sum_{i\in[N_0], s_i<t} w_i
    \bigl(1-e^{-(t-s_i)/\tau}\bigr).
\end{equation}
Its spike time is $t^* = \mathrm{inf}\{t: u(t) = \vartheta\}$.
If $t^*<\infty$, its exclusive causal set $\mathcal{C}$ satisfies $\sum_{i\in\mathcal{C}}w_i>\vartheta$ and
\begin{equation}
    t^*
    =
    \tau
    \log
    \left(
        \frac{\sum_{i\in C} w_i e^{s_i/\tau}}
             {\sum_{i\in C} w_i - \vartheta}
    \right) \,.
\end{equation}
\end{definition}
\end{highlightbox}
\begin{highlightbox}
\begin{definition}[nLIF with Dirac-$\delta$ synapses]\label{def:instant}
    Let $N_0 \in \mathbb{N}$, $\vartheta > 0$ be the spike threshold, $w \in \mathbb{R}^{N_0}$ the synaptic weights, and $s \in (\mathbb{R}\cup\{\infty\})^{N_0}$ the input spike times. 
    Then the membrane potential $u(t) \in \mathbb{R}$ of a \textbf{nLIF neuron with Dirac-$\pmb{\delta}$ synapses} at time $t \in \mathbb{R}$ is given by:
\begin{equation}
    u(t)
    =
    \sum_{i\in[N_0], s_i\leq t} w_i \,.
\end{equation}
Its spike time is $t^* = \mathrm{inf}\{t: u(t)  \geq \vartheta\}$.
Since the membrane potential can cross the positive threshold from below only upon the arrival of a presynaptic spike, it follows, when $t^*<\infty$, that the neuron's inclusive causal set $\mathcal{C}$ is nonempty and
\begin{equation}
    t^*
    =
    \max_{i\in \mathcal{C}} s_i .
\end{equation}
\end{definition}
\end{highlightbox}
\begin{highlightbox}
\begin{definition}[LIF with Dirac-$\delta$ synapses]\label{def:dirac}
    Let $N_0 \in \mathbb{N}$, $\tau \in \mathbb{R}^+$ be the membrane time constant, $\vartheta > 0$ the spike threshold, $w \in \mathbb{R}^{N_0}$ the synaptic weights, and $s \in (\mathbb{R}\cup\{\infty\})^{N_0}$ the input spike times. 
    Then the membrane potential $u(t) \in \mathbb{R}$ of a \textbf{LIF neuron with Dirac-$\pmb{\delta}$ synapses} at time $t \in \mathbb{R}$ is given by:
\begin{equation}
    u(t)
    =
    \sum_{i\in[N_0], s_i\leq t} w_i
    e^{-(t-s_i)/\tau} .
\end{equation}
Its spike time is $t^* = \mathrm{inf}\{t: u(t)  \geq \vartheta\}$.
Since the membrane potential can cross the positive threshold from below only upon the arrival of a presynaptic spike, it follows, when $t^*<\infty$, that the neuron's inclusive causal set $\mathcal{C}$ is nonempty and
\begin{equation}
    t^*
    =
    \max_{i\in \mathcal{C}} s_i .
\end{equation}
\end{definition}
\end{highlightbox}
The ReLU-type response model has recently been used to derive theoretical results \cite{neuman2024stable,singh2023expressivity,singh2026polyhedral} and to map ReLU neural networks to time-to-first-spike neural networks \cite{stanojevic2023exact,stanojevic2024high}.
Similarly, the nLIF models have been used to train SNNs with exact error backpropagation \cite{mostafa2017supervised,goltz2021fast,kheradpisheh2020temporal}.

%% file: content/proof_general.tex
Before deriving bounds, we first introduce a useful alternative representation of causal subnetworks (\cref{def:binarymatrices}).
\begin{highlightbox}
\begin{definition}\label{def:binarymatrices}
Assume a feedforward SNN $\Phi$ with depth $L$, input dimension $N_0$, and layer widths $N_1,\dots,N_L$.
For an input $b\in\R^{N_0}$ and $\diamond\in\{<,\leq\}$, we represent the causal subnetwork of $\Phi$ by binary matrices
\begin{align*}
    B_\diamond^{(\ell)}(b)
    \in\{0,1\}^{N_\ell\times N_{\ell-1}},
    \qquad \ell\in[L],
\end{align*}
where
\begin{align*}
    \bigl(B_\diamond^{(\ell)}(b)\bigr)_{ij}
    :=
    \begin{cases}
        1, & \text{if neuron $i$ in layer $\ell$ belongs to the causal subnetwork and }
             j\in\mathcal C_{i,\diamond}^{(\ell)}(b),\\
        0, & \text{otherwise}.
    \end{cases}
\end{align*}
Hence, the tuple $\bigl(B_\diamond^{(1)}(b),\dots,B_\diamond^{(L)}(b)\bigr)$ represents the causal subnetwork of $\Phi(b)$.
\end{definition}
\end{highlightbox}
From \cref{def:binarymatrices}, since every causal piece corresponds to a unique causal subnetwork, it immediately follows that the number of causal pieces can be bounded by the number of realisable tuples $\bigl(B_\diamond^{(1)},\dots,B_\diamond^{(L)}\bigr)$ (\cref{prop:trivialBound}).
\begin{highlightbox}
\begin{proposition}[Worst-case upper bound]
\label{prop:trivialBound}
Let $\diamond\in\{<,\leq\}$ and let $\Phi$ be a feedforward single-spike neural network with $L$ layers, input dimension $N_0$, widths $N_1,N_2,\ldots,N_L$, and constant weights $W^{(\ell)} \in \mathbb{R}^{N_\ell \times N_{\ell-1}}$ as well as constant neuron parameters. Let $\mathfrak{C}(\Phi)$ be the number of causal pieces of $\Phi$ over $\mathbb{R}^{N_0}$. Then
\begin{equation}
\mathfrak{C}_{\diamond}(\Phi)
\leq
\prod_{\ell=1}^{L}
2^{N_{\ell-1} N_\ell} = 2^{\sum_{\ell=1}^{L} N_{\ell-1} N_\ell}\,,
\label{eq:upper_bound}
\end{equation}
and for $N = \max\left\{N_0, ..., N_L\right\}\,,$ 
\begin{equation}
\mathfrak{C}_{\diamond}(\Phi) \leq 2^{LN^2} \,.
\label{eq:upper_bound_width}
\end{equation}
\end{proposition}
\end{highlightbox}

In Proposition~\ref{prop:trivialBound}, we assumed that $B^{(\ell)}$ can be an arbitrary binary matrix.
However, this is not the case due to the causal nature of spikes, which imposes a certain structure on the binary matrices.
As Lemma~\ref{lemma:lonesum} shows, the $B^{(\ell)}$ are so-called lonesum matrices, i.e., matrices that can be uniquely reconstructed from their row-and column-sum-vectors \cite{ryser1957combinatorial}.
\begin{highlightbox}
\begin{lemma}[Lonesum]\label{lemma:lonesum}
    The binary matrices $\left(B^{(1)}, ..., B^{(L)}\right)$ (see \cref{def:binarymatrices}), representing a causal piece in a $L$-layer SNN where all neurons in a layer receive the same input spikes, are lonesum matrices.
\end{lemma}
\end{highlightbox}
\begin{proof}
    A matrix is lonesum if and only if it does not contain either of the two submatrices \cite{ryser1957combinatorial,brewbaker2008combinatorial}
    \begin{equation}\label{eq:submatrices}
        \begin{pmatrix}
1 \ \ 0\\ 0\ \ 1
\end{pmatrix}\quad 
\text{and}\quad\begin{pmatrix}
0 \ \ 1\\  1 \ \ 0
\end{pmatrix} \,.
    \end{equation}
We show that neither pattern can occur in the activity matrices $B^{(\ell)}$.
Assume first that $B^{(\ell)}$ contains the first pattern. Then there exist two postsynaptic neurons $i_1,i_2$ in layer $\ell$ and two
presynaptic neurons $j_1,j_2$ in layer $\ell-1$ such that
\begin{equation}
    B^{(\ell)}_{i_1j_1}=1,\qquad
    B^{(\ell)}_{i_1j_2}=0,\qquad
    B^{(\ell)}_{i_2j_1}=0,\qquad
    B^{(\ell)}_{i_2j_2}=1.
\end{equation}
where we omit the $\diamond$ index for readability. Suppose that neuron $i_1$ spikes before neuron $i_2$, i.e.,
\begin{equation}
    t^{(\ell)}_{i_1}<t^{(\ell)}_{i_2}.
\end{equation}
Since $B^{(\ell)}_{i_1j_1}=1$, the spike of neuron $j_1$ belongs to the causal set of neuron $i_1$, and hence
\begin{equation}
    t^{(\ell-1)}_{j_1}
    \diamond t^{(\ell)}_{i_1}
    < t^{(\ell)}_{i_2},
\end{equation}
where $\diamond\in\{<,\leq\}$ depending on the type of causal set. %
Therefore, the spike of neuron $j_1$ must also belong to the causal set of neuron $i_2$, implying
\begin{equation}
    B^{(\ell)}_{i_2j_1}=1,
\end{equation}
which contradicts the assumed submatrix.
Conversely, if neuron $i_2$ spikes before neuron $i_1$, then
$B^{(\ell)}_{i_2j_2}=1$ implies
\[
    t^{(\ell-1)}_{j_2}
    \diamond t^{(\ell)}_{i_2}
    < t^{(\ell)}_{i_1},
\]
where $\diamond\in\{<,\leq\}$ depending on the type of causal set.
Therefore $B^{(\ell)}_{i_1j_2}=1$, again yielding a
contradiction. 
If the two postsynaptic neurons spike simultaneously, their causal sets are identical, and hence the above pattern is also
impossible.
The same argument, with $j_1$ and $j_2$ exchanged, excludes the
second forbidden submatrix. 
Finally, a neuron that does not spike is assigned the empty causal set and therefore corresponds to an all-zero row, which cannot form either forbidden submatrix.
Hence $B^{(\ell)}$ contains neither forbidden $2\times2$ pattern and is therefore lonesum.
\end{proof}
\begin{remark}
    Every lonesum matrix also determines a nested causal-set assignment. 
Suppose that the causal sets $C_i$ and $C_k$ corresponding to rows $i$, $k$ of the lonesum matrix were incompatible. 
Then there would exist $j_1\in C_i\setminus C_k$ and $j_2\in C_k\setminus C_i$, so that the submatrix formed by rows $i,k$ and columns $j_1,j_2$ would be one of the two forbidden $2\times2$ matrices. 
Hence, if the matrix has $m$ rows, there exists a permutation $\pi\in S_m$ such that $C_{\pi(1)}\subseteq\cdots\subseteq C_{\pi(m)}$.
Zero rows correspond to the empty causal set.
This does not mean that a deep SNN can realise any possible assignment of lonesum matrices in practice.
\end{remark}

As shown in \cite[Theorem 3]{brewbaker2008combinatorial}, the number of lonesum matrices of dimension $n\times k$ is exactly given by the poly-Bernoulli number $\mathcal{B}_n^{-k}$.
Thus, we can use the poly-Bernoulli numbers to bound the number of causal pieces.
In addition, to estimate how the bound scales in general, we use an alternative bound for the number of lonesum matrices (\cref{lemma:poly}).
\begin{highlightbox}
\begin{lemma}\label{lemma:poly}
Let $\mathcal{B}_n^{-k}$ be the poly-Bernoulli numbers, which is exactly the number of lonesum matrices with dimension $n\times k$ for $n,k \in \mathbb{N}^+$. Then
\begin{equation}
    \mathcal{B}_n^{-k} \leq k! (k+1)^n \,.
\end{equation}
\end{lemma}
\end{highlightbox}
\begin{proof}
    Let $M\in\{0,1\}^{n\times k}$ be lonesum, and let $c_j=\sum_{i=1}^n M_{ij}$ denote the column sums. 
    Further, let $\sigma$ be a permutation of the columns such that 
    $c_{\sigma(1)}\ge c_{\sigma(2)}\ge \cdots \ge c_{\sigma(k)}$.
    
    We start with the column with the largest column-sum.
    Assume there is a row $r$ for which $M_{r\sigma(1)} = 0$.
    We claim that in this case, $M_{rj} = 0$ $\forall j \neq \sigma(1)$.
    If there is an $\ell$ such that $M_{r\ell} = 1$, we have two cases:
    \begin{enumerate}
        \item There is at least one $i$ such that $M_{i\sigma(1)} = 1$ and $M_{il} = 0$. However, then $M$ contains at least one of the two submatrices (Eq.~\ref{eq:submatrices}) and is not lonesum, contradicting our initial assumption.
        \item Thus, for all $i \in \mathcal{I} = \{j: M_{j\sigma(1)} = 1 \}$, it has to follow that $M_{il} = 1$. 
        But then $c_\ell \geq |\mathcal{I}| + M_{r\ell} > |\mathcal{I}| = c_{\sigma(1)}$, contradicting our initial assumption that $\sigma(1)$ is the column with the largest column-sum.
    \end{enumerate}
    By removing column $\sigma(1)$ (through which the matrix stays lonesum), this argument can be repeated iteratively for all other columns.
    Thus, permuting the columns of $M$ using $\sigma$ leads to a matrix where the rows take the form $R_s = (1,\ldots,1,0,\ldots,0)$, i.e., $s$ times $1$ followed by only $0$ entries, with $s \in [0, k]$.
    We can then construct all lonesum matrices of dimension $n\times k$ in the following way:
    first, assign a $R_s$ to every row. There are, in total, $(k+1)^n$ such combinations.
    We then apply column permutations, leading to $k!(k+1)^n$ possible lonesum matrices.
    Compared to $\mathcal{B}_n^{-k}$, we potentially over-count matrices where we permuted identical columns, hence $\mathcal{B}_n^{-k} \leq k!(k+1)^n$.
\end{proof}

From \cref{lemma:lonesum,lemma:poly} follows immediately our improved bound (\cref{thm:lonesum_bound}).
\begin{highlightbox}
\begin{theorem}[Combinatoric upper bound]\label{thm:lonesum_bound}
    Let $\diamond\in\{<,\leq\}$ and let $\Phi$ be a feedforward single-spike neural network with $L$ layers, input dimension $N_0$, widths $N_1,N_2,\ldots,N_L$, and constant weights $W^{(\ell)} \in \mathbb{R}^{N_\ell \times N_{\ell-1}}$ as well as constant neuron parameters. Assume that every neuron of a layer $\ell$ always receives the same presynaptic spike times as input, as is, e.g., the case in SNNs with no synaptic delays. 
    Then
\begin{equation}\label{eq:polyBound}
    \mathfrak{C}_{\diamond}(\Phi)
\leq
\prod_{\ell=1}^{L}
\mathcal{B}_{N_\ell}^{-N_{\ell-1}} \,,
\end{equation}
    where $\mathcal{B}_n^{-k}$ are the poly-Bernoulli numbers as defined in Eq.~\ref{eq:polybernoulli}.
    For $N = \max\left\{N_0, ..., N_L\right\},$ it follows that 
\begin{equation}
\mathfrak{C}_{\diamond}(\Phi) \leq (N+1)^{2NL} \,.
\label{eq:upper_bound_width_lonesum}
\end{equation}
\end{theorem}
\end{highlightbox}
\begin{proof}
Following \cref{lemma:lonesum}, the number of causal pieces can be bounded by the number of lonesum matrices of dimension $n\times k$, which is given by the poly-Bernoulli numbers $\mathcal{B}^{(-k)}_n$ \cite{brewbaker2008combinatorial}, see Eq.~\ref{eq:polybernoulli}.
This immediately proves Eq.~\ref{eq:polyBound}.
The poly-Bernoulli numbers can further be upper bounded using Lemma~\ref{lemma:poly}:
$\prod_{\ell=1}^{L}
\mathcal{B}_{N_\ell}^{-N_{\ell-1}} \leq \prod_{\ell=1}^{L}
N_{\ell-1}!(N_{\ell-1}+1)^{N_{\ell}} \leq \prod_{\ell=1}^{L} N!(N+1)^N \leq (N+1)^{2NL}$.
\end{proof}

\begin{remark}
    The poly-Bernoulli numbers arise from counting assignments of nested
    causal sets within a layer. Consider $n$ neurons receiving
    the same $k$ spike times. Their distinct nonempty causal
    sets form a chain
    \begin{align}
        C_1 \subsetneq \cdots \subsetneq C_m,
        \qquad C_0=\varnothing. \notag
    \end{align}
    The neurons partition into $m$ nonempty classes according
    to their causal sets $C_1,\ldots,C_m$, together with a possibly empty
    class assigned $C_0$. The neurons of the preceding layer partition into the
    nonempty increments
    \begin{align}
        D_r=C_r\setminus C_{r-1}
        \qquad\text{for every }r\in[m], \notag
    \end{align}
    together with the possibly empty set $D_0=[k]\setminus C_m$.
    Thus, these neurons are assigned to disjoint increments, rather
    than to the overlapping causal sets themselves.
    Each partition consists of $m$ ordered nonempty sets and one possibly
    empty set. Adding a distinguished element to the latter shows that
    the numbers of choices are
    \begin{align}
        m!\left\{\begin{matrix}n+1\\m+1\end{matrix}\right\}
        \quad\text{and}\quad
        m!\left\{\begin{matrix}k+1\\m+1\end{matrix}\right\}, \notag
    \end{align}
    respectively. Each pair uniquely determines a causal-set assignment,
    since $C_r=D_1\cup\dots\cup D_r$. Summing their products over
    $m\in\{0,\ldots,\min(n,k)\}$ gives $B_n^{(-k)}$.
\end{remark}

%% file: content/hyperplanes.tex
In the following, we derive upper bounds for the number of causal pieces using hyperplane arrangement arguments.
We first state the assumptions a spiking neuron model has to satisfy for our results to apply (\cref{ass:model}).
We then continue by deriving a general upper bound for neurons models satisfying this assumption (\cref{thm:hyper_general}), before deriving particular upper bounds for the nLIF neuron model with exponential synapses, the ReLU-type response model, and nLIF/LIF model with Dirac-$\delta$ synapses (\cref{thm:nLIF_hyperplane,thm:ReLU_hyperplane,thm:nLIF_Dirac,thm:LIF}).
Especially for neuron models with Dirac-$\delta$ synapses, due to the constraint of output spikes having to coincide with input spikes, we are able to derive much tighter bounds.

Generally, the bounds derived using hyperplane arguments complement the general combinatoric bound (\cref{thm:lonesum_bound}).
However, in terms of scaling behaviour with maximum width, input dimension, and depth, the hyperplane bounds provide far tighter scaling than the combinatoric ones (see \cref{tab:width-asymptotics}).

\subsection{Spiking neuron model assumptions}
We collect sufficient properties of a spiking neuron model that allows bounding the number of causal pieces via hyperplane arrangement arguments. To this end we introduce some additional notation and terminology. 

\begin{highlightbox}
\begin{definition}[Coordinate map]\label{def:coordmap}
A \textbf{coordinate map} is a strictly increasing homeomorphism $\rho\colon\R\to D$, where $D\subseteq\R$ is a nonempty open interval. We use the convention $\rho^{-1}(\infty):=\infty$.
\end{definition}
\end{highlightbox}
The principal examples are time coordinates $\rho(u)=u$ or, for $\tau\in(0,\infty)$, exponential coordinates $\rho(u)=e^{\frac{u}{\tau}}$ and $\rho(u)=-e^{-\frac{u}{\tau}}$, with ranges $\R$, $(0,\infty)$, and $(-\infty,0)$ respectively. The convention $\rho(\infty):=\infty$ is used to handle infinite spike times. 
\begin{highlightbox}
\begin{definition}[Weak order domains]
    Let $D\subseteq\R$, $J\subseteq\N$ finite and nonempty, and $q\in[|J|]$. An ordered partition $G=(G_1,\dots,G_q)$ of $J$ is called a \textbf{weak order} and
    \begin{align*}
        Q_G:=\{x\in D^J\colon \big((x_j=x_k \text{ if }\exists r\in[q]\colon j,k\in G_r)\wedge(x_j<x_k \text{ if }\exists r\in[q-1]\colon j\in G_r, k\in G_{r+1})\big)\ \forall j,k\in J\}
    \end{align*}
    is called a \textbf{weak order domain}.
\end{definition}
\end{highlightbox}
For example, let $J = \{0,1,2,3\}$, $D=(0,\infty)$, and $x = (\frac{1}{2},\frac{1}{2}, \frac{1}{4}, 1)$. 
Since $x_2 < x_0=x_1 < x_3$, the weak order associated with $x$ is $G=(\{2\},\{0,1\},\{3\})$ and therefore $x\in Q_G$.

\begin{highlightbox}
\begin{definition}[Transformed spike-time maps]\label{def:trans_spike_time}
Let $L\in\N$, let $N_0,\dots,N_L\in\N$ be the layer widths, and let $\rho\colon\R\to D \subseteq{\mathbb{R}}$ be a coordinate map.
For neuron $i\in[N_\ell]$ in layer $\ell\in[L]$, a \textbf{transformed spike-time map} is a map
\begin{align*}
    f_i^{(\ell)}\colon (D\cup\{\infty\})^{N_{\ell-1}}&\to D\cup\{\infty\} \,,
\end{align*}
that satisfies, for every $b\in\R^{N_0}$, that
\begin{align*}
    f_i^{(\ell)}(\rho(t^{(\ell-1)}(b)))=\rho(t^{(\ell)}_i(b)) \,,
\end{align*}
where $\rho$ is applied componentwise, and we write
\begin{align*}
        t^{(r)}(b):=(t_1^{(r)}(b),\dots, t_{N_r}^{(r)}(b))\in (\R\cup\{\infty\})^{N_r} \,,
\end{align*}
for the vector of spike times in layer $r\in[L]$, with $t^{(0)}(b):=b$.
For $z\in D^{N_0}$, we define recursively
\begin{align*}
    x^{(0)}(z):=z\,,
    \qquad
    x^{(\ell)}(z):=(f_i^{(\ell)}(x^{(\ell-1)}(z)))_{i\in [N_\ell]} \,.
\end{align*}
\end{definition}
\end{highlightbox}
Fixing a weak order $G$ isolates a region of spike-time space $Q_G$ on which the relative arrival order of all finite presynaptic spikes is fixed.
\begin{highlightbox}
\begin{definition}[Local restriction \& Causal-set maps]\label{def:causal_maps}
For neuron $i\in[N_\ell]$ in layer $\ell\in[L]$, a nonempty $J\subseteq[N_{\ell-1}]$, a weak order $G$ of $J$, and $y=(y_j)_{j\in J}\in Q_G$, where $y_j$ denotes the transformed spike
time of neuron $j$, we define the \textbf{local restriction}
\begin{align*}
    f_{i,G}^{(\ell)}\colon Q_G\to D\cup\{\infty\} \,,
\end{align*}
by 
\begin{align*}
    f_{i,G}^{(\ell)}(y):=f_i^{(\ell)}(\tilde{y}),\ \text{where}\ \tilde{y}_j:=\begin{cases}y_j, & j\in J \,,\\ \infty, & j\in [N_{\ell-1}]\backslash J \,.\end{cases}.
\end{align*}
Let $\diamond\in\{<,\leq\}$. The associated \textbf{causal-set map} $\mathcal{C}^{(\ell)}_{i,G,\diamond}\colon Q_G\to2^J$ is given by
\begin{align*}
    \mathcal{C}^{(\ell)}_{i,G,\diamond}(y):=\begin{cases}
        \{j\in J\colon y_j\diamond f_{i,G}^{(\ell)}(y)\}, & f_{i,G}^{(\ell)}(y) < \infty \,,\\
        \emptyset, & f_{i,G}^{(\ell)}(y)=\infty \,.
    \end{cases}.
\end{align*}
For $J=\{j\in[N_{\ell-1}]\colon t^{(\ell-1)}_j(b)<\infty\}\neq\emptyset$ and a weak order $G$ of $J$ with $\rho(t_J^{(\ell-1)}(b))\in Q_G$, it is related to causal sets w.r.t.\@ untransformed input $b\in\R^{N_0}$ via  
\begin{align*}
    \mathcal{C}^{(\ell)}_{i,G,\diamond}(\rho(t_J^{(\ell-1)}(b)))=\mathcal{C}^{(\ell)}_{i,\diamond}(b).
\end{align*}
\end{definition}
\end{highlightbox}

In the following, a polyhedral set is an intersection of finitely many affine hyperplanes and open or closed affine halfspaces.
Due to the mechanism of neurons spiking when a potential reaches a certain threshold, causal pieces can often be described via such intersections of halfspaces, potentially after some coordinate transformation. This enables us to represent causal pieces as unions of polyhedral sets respectively their preimage under the coordinate transformation.
\cref{ass:model} states sufficient requirements to allow deriving upper bounds via hyperplane arrangement arguments.

\begin{highlightbox}
\begin{assumption}\label{ass:model}
    Let $\diamond\in\{<,\leq\}$, and let $\Phi$ be a feedforward single-spike network with depth $L\in\N$ and layer widths $N_0,\dots,N_L\in\N$. Suppose there exist \textbf{transformed spike-time maps} $f_i^{(\ell)}$ with \textbf{common coordinate map} $\rho\colon\R\to D\subseteq\R$ such that, for every neuron $i\in[N_\ell]$, layer $\ell\in[L]$, and $y\in(D\cup\{\infty\})^{N_{\ell-1}}$,
    \begin{align}\label{eq:nfbfi}
        f_i^{(\ell)}(y)\geq \min_{j\in[N_{\ell-1}]}y_j \,.
    \end{align}
    Moreover, there exist \textbf{nondecreasing functions $p\colon\N\to\N$ and $h\colon\N\to\N_0$} for which the following holds.
    Fix $\ell\in[L]$, $i\in[N_\ell]$, a nonempty $J\subseteq[N_{\ell-1}]$, and a weak order $G$ on $J$. 
    Then there exist
    \begin{itemize}
        \item a partition $\mathcal{A}$ of $Q_G$ into at most $p(|J|)$ nonempty polyhedral sets,
        \item affine functions $g_1,\dots g_H\colon \R^J\to\R$ for some $H\leq h(|J|)$,
    \end{itemize} 
    such that:
    \begin{enumerate}[(i)]
        \item \textbf{Local affine structure:} For every $A\in\mathcal{A}$, the map $\mathcal{C}_{i,G,\diamond}^{(\ell)}$ is constant on $A$, and either $f^{(\ell)}_{i,G}\equiv\infty$ on $A$, or $f^{(\ell)}_{i,G}\colon A\to D$ is affine.
        
        \item \textbf{Hyperplane separation:} For every $\sigma\in\{-1,0,1\}^H$, the set
        \begin{align*}
            \{y\in Q_G\colon\operatorname{sgn} g_k(y)=\sigma_k\ \forall k\in[H]\}
        \end{align*}
        is empty or contained in one set $A\in\mathcal{A}$.
        \item \textbf{Symmetry:} One of the following two conditions holds throughout the network:
        \begin{description}
            \item[(T)] $D=\R$ and $g_k(y+c\mathbf{1}_J)=g_k(y)$ for all $k\in[H]$, $y\in\R^J$, $c\in\R$, where $\mathbf{1}_J:=(1)_{j\in J}$. Moreover, for every $y\in Q_G$ with $f_{i,G}^{(\ell)}(y)<\infty$ and every $c\in\R$, it holds that
            \begin{align*}
                f_{i,G}^{(\ell)}(y+c\mathbf{1}_J) = f_{i,G}^{(\ell)}(y) + c.
            \end{align*}
            \item[(H)]There exists a fixed $s_H\in\{-1,1\}$ such that
$D=s_H(0,\infty)$, where $s_H(0,\infty):=\{s_H r:r>0\}$. 
Moreover, $g_k$ is linear for every $k\in[H]$. 
For every $y\in Q_G$ with $f_{i,G}^{(\ell)}(y)<\infty$ and every $\lambda > 0$, it holds that
            \begin{align*}
                f_{i,G}^{(\ell)}(\lambda y)=\lambda f_{i,G}^{(\ell)}(y).
            \end{align*}
        \end{description}
    \end{enumerate}
\end{assumption}
\end{highlightbox}
The condition in \cref{eq:nfbfi} ensures that no output spike occurs
before the earliest input spike.
We then consider a single neuron and fix a weak ordering $G$ of its input spike times. 
The weak-order domain $Q_G$ is the subset of the transformed input space on which this ordering is satisfied.

Fixing the input order does not, however, determine the causal set.
The domain $Q_G$ is therefore further partitioned into polyhedral
regions $A\in\mathcal A$ on each of which the causal set is fixed.
For example, if
\begin{align*}
    G=(\{2\},\{4\},\{1,3\}) \,,
\end{align*}
then
\begin{align*}
    y_2<y_4<y_1=y_3 \,,
\end{align*}
and the possible nonempty causal sets are
\begin{align*}
    \{2\},\qquad
    \{2,4\},\qquad
    \{2,4,1,3\} \,.
\end{align*}
The affine functions $g_1,\dots,g_H$ provide hyperplanes whose sign
patterns distinguish the regions of $\mathcal A$.

The three conditions have the following interpretation.
\textbf{(i)} On every region $A\in\mathcal A$, the causal set is
constant. Moreover, whenever the neuron fires, its transformed output
spike time is an affine function of its transformed input spike times.
\textbf{(ii)} The signs of $g_1,\dots,g_H$ determine the region
$A\in\mathcal A$: two points with the same sign pattern must lie in the
same region.
\textbf{(iii)} The neuron model satisfies one of two global symmetries.
In case~(T), shifting all input spike times by the same amount shifts the
output spike time by the same amount while leaving the switching
conditions unchanged. In case~(H), the corresponding symmetry is a
common positive scaling of the transformed spike times. It should be noted that these symmetry conditions only reduce the effective dimension of each layer by one, and without them the bound in \Cref{thm:hyper_general} still holds with $d_\ell=\min_{0\leq r\leq\ell}N_r$ instead of $d_\ell:=\min_{0\leq r\leq\ell} (N_r-1)$.

The symmetry requirement is related to the monotone scaling property of \cite[Definition 3.1(iii)]{hundrieser2025universal}, which requires spike-train maps to commute with every strictly increasing bijection of the time domain. For the coordinate maps used in our particularisations, both (T) and (H) express equivariance under common translations of the original spike times. The monotone scaling property is stronger, as it also allows arbitrary changes to the gaps between successive spikes while preserving their order. When required of every neuron’s spike-time map, it implies that causal sets depend only on the weak order of incoming spike times. This holds for our nLIF model with Dirac synapses, whereas ramp-response and exponential-synapse nLIF neurons generally retain dependence on the time gaps and only satisfy the weaker translation symmetry.

\subsection{General hyperplane bound}

We first provide a general hyperplane bound before specifying it for particular neuron models.
For stating the bound, we require \cref{lem:sign_counting}.

\begin{highlightbox}
\begin{lemma}\label{lem:sign_counting}
    Let $d,H\in\N_0$ and let $g_1,\dots,g_H\colon\R^d\to\R$ be affine functions. Then
    \begin{align*}
        |\{(\operatorname{sgn}g_1(x),\dots,\operatorname{sgn}g_H(x))\colon x\in\R^d\}|\leq \sum_{j=0}^d 2^j{H \choose j}.
    \end{align*}
\end{lemma}
\end{highlightbox}
\begin{proof}
    We assume w.l.o.g.\@ that the functions $g_1,\dots,g_H$ are non-constant and have distinct zero hyperplanes, as the alternative would only reduce the cardinality of the set. The considered sign patterns correspond bijectively to the nonempty faces (of any dimension) of the arrangement of the zero hyperplanes of the functions, in the sense of \cite{EdelsbrunnerORourkeSeidel1986}. By \cite[Lemma 2.5(i)]{EdelsbrunnerORourkeSeidel1986} the number of $k$-dimensional faces is bounded by $\sum_{j=d-k}^d{j \choose d-k}{H \choose j}$. Summing over the dimensions yields
    \begin{align*}
        \sum_{k=0}^d\sum_{j=d-k}^d{j \choose d-k}{H \choose j}=\sum_{j=0}^d{H\choose j}\sum_{r=0}^j{j\choose  r}=\sum_{j=0}^d 2^j {H\choose j}.
    \end{align*}
\end{proof}

With \cref{lem:sign_counting} and under \cref{ass:model}, we derive the general hyperplane bound (\cref{thm:hyper_general}).
\begin{highlightbox}
\begin{theorem}\label{thm:hyper_general}
    Let $\diamond\in\{<,\leq\}$ and let $\Phi$ be a feedforward single-spike network with depth $L$ and widths $N_0,\dots,N_L$. If it satisfies \cref{ass:model} with $p\colon\N\to\N$, $h\colon\N\to\N_0$, we have 
    \begin{align*}
        \mathfrak{C}_{\diamond}(\Phi) \leq \prod_{\ell=1}^L \sum_{j=0}^{d_{\ell-1}} 2^j{{N_{\ell-1}\choose 2}\choose j}\min\left\{p(N_{\ell-1})^{N_\ell},\ \sum_{j=0}^{d_{\ell-1}}2^j{N_\ell h(N_{\ell-1})\choose j}\right\},
    \end{align*}
    where $d_\ell:=\min_{0\leq r\leq\ell} (N_r-1)$, $\ell\in\{0,\dots,L\}$. 
\end{theorem}
\end{highlightbox}
\begin{proof}
    We iteratively construct partitions $\cP_0,\dots,\cP_L$ of $D^{N_0}$ into nonempty polyhedral sets. We require the following properties to hold for every $r\in\{0,\dots,L\}$ and every region $P\in\cP_r$:
    \begin{enumerate}[(a)]
        \item \label{ind_ass_a} Causal sets are constant on $P$, i.e.\@ for all $s\in[r]$, $i\in[N_s]$, $z,z'\in P$, we have
        \begin{align*}
            \cC_{i,\diamond}^{(s)}(\rho^{-1}(z))=\cC_{i,\diamond}^{(s)}(\rho^{-1}(z')) \,,
        \end{align*}
        where $\rho^{-1}$ is applied componentwise.
        \item \label{ind_ass_b} The index sets of neurons that fire in finite time are constant on $P$, i.e.\@ for all $z,z'\in P$, we have
        \begin{align*}
           \{i\in[N_r]\colon x_i^{(r)}(z)<\infty\} = \{i\in[N_r]\colon x_i^{(r)}(z')<\infty\}=:J_r(P) \,.
        \end{align*}
        If $J_r(P)\neq\emptyset$, the map $x^{(r)}_{J_r(P)}:=(x^{(r)}_j)_{j\in J_r(P)}$ is affine on $P$, and $x^{(r)}_{J_r(P)}(P)$ has affine dimension\footnote{The affine dimension of a set $S$ is the smallest integer $n=:\dim S$ such that $S$ is contained in an $n$-dimensional affine subspace.}  at most $d_r+1$.
        \item \label{ind_ass_c} In case of condition (T), it holds, for every $c\in\R$, that
        \begin{align*}
            P+c\mathbf{1}&=P \,,
        \end{align*}
        and, for every $j\in J_r(P)$, $z\in P$, that
        \begin{align*}
            x^{(r)}_{j}(z+c\mathbf{1})&=x^{(r)}_{j}(z)+c \,.
        \end{align*}
        In case of condition (H), it holds, for every $\lambda\in(0,\infty)$, that
        \begin{align*}
            \lambda P&=P \,,
        \end{align*}
        and, for every $j\in J_r(P)$, $z\in P$, that
        \begin{align*}
            x^{(r)}_{j}(\lambda z)&=\lambda x^{(r)}_{j}(z)\,.
        \end{align*}
    \end{enumerate}
    We start the induction with
    \begin{align*}
        \cP_0:=\{D^{N_0}\}\,.
    \end{align*}
    The required properties hold since $x^{(0)}$ is the identity and $N_0=d_0+1$.
    For the induction step we assume that these properties hold for $\ell-1$ and observe the following. Let $\ell\in[L]$, $P\in\cP_{\ell-1}$, and write $J:=J_{\ell-1}(P)$ for simplicity of notation. If $J=\emptyset$ we can simply include $P$ in $\cP_\ell$ since the required conditions continue to hold and no subdivision is needed. Otherwise, we define the set
    \begin{align*}
        Y_P:=x^{(\ell-1)}_{J}(P)\,.
    \end{align*}
    We note that $Y_P$ has affine dimension at most $d_{\ell-1}+1$ due to induction assumption \eqref{ind_ass_b}, and that $Y_P+c\mathbf{1}=Y_P$ or $\lambda Y_P=Y_P$ respectively by induction assumption \eqref{ind_ass_c}. Before we count sign patterns we can reduce the dimension by one if we account for the respective invariances.
    Specifically, we argue that the section
    \begin{align*}
        \widehat{Y}_P:=\begin{cases}
            Y_P\cap\{y\in D^{J}\colon\sum_{j\in J}y_j=0\}, & \text{if (T)}\\
            Y_P\cap\{y\in D^{J}\colon s_H\sum_{j\in J}y_j=1\}, & \text{if (H)}
        \end{cases}
    \end{align*}
    realises the same sign patterns of the coordinate comparisons $y_j-y_k$ as $Y_P$ and has affine dimension at most $d_{\ell-1}$. We can, for every $y
   \in Y_P$, define its normalisation $\widehat{y}\in \widehat{Y}_P$ as
    \begin{align*}
        \widehat{y}:=\begin{cases}
            y - \tfrac{1}{|J|}(\sum_{j\in J}y_j )\mathbf{1}, & \text{if (T)}\\
            \frac{y}{s_H\sum_{j\in J}y_j}, & \text{if (H)}
        \end{cases}
    \end{align*}
    and note that $\operatorname{sgn}(\widehat{y}_i-\widehat{y}_j)=\operatorname{sgn}(y_i-y_j)$ for $y\in Y_P$, $i,j\in J$.
    Moreover, we note that 
    $y\mapsto \sum_{j\in J}y_j$ respectively $y\mapsto s_H\sum_{j\in J}y_j$ are nonconstant functionals on $Y_P$, which implies 
    \begin{align*}
        \dim(\widehat{Y}_P)\leq \dim(Y_P)-1\leq d_{\ell-1}\,.
    \end{align*}
    The number of possible weak orders of a vector $y=(y_j)_{j\in J}$ is determined by the ${|J| \choose 2}\leq{N_{\ell-1} \choose 2}$ comparisons $g_{jk}(y):=y_j-y_k$ with $j,k\in J$ and $j<k$. Therefore we can apply \Cref{lem:sign_counting} to the affine hull of $\widehat{Y}_P$ to bound the number of possible weak orders on $\widehat{Y}_P$ (and, by the above, on $Y_P$) by 
    \begin{align*}
        \sum_{j=0}^{d_{\ell-1}}2^j{{|J|\choose 2}\choose j}\,.
    \end{align*}
    For each weak order $G$ of $J$ realised on $Y_P$, we define
    \begin{align*}
        P_G:=\{z\in P\colon x^{(\ell-1)}_{J}(z)\in Q_G\}\,.
    \end{align*}
    These sets form a polyhedral partition of $P$.
    We fix one such $G$. For $i\in[N_\ell]$, let $\mathcal{A}_i$ and $g_{i,1},\dots,g_{i,h_i}$ with $h_i\leq h(|J|)$ be the partition, respectively affine functions, guaranteed by Assumption~\ref{ass:model}. We can refine $P_G$ in one of the following two ways.
    The first option is taking, for every choice of $(A_1,\dots,A_{N_\ell})\in\prod_{i\in[N_\ell]}\mathcal{A}_i$, the set
    \begin{align*}
        \{z\in P_G\colon x^{(\ell-1)}_{J}(z)\in\mathcal A_i \ \forall i\in[N_\ell]\}\,,
    \end{align*}
    if it is nonempty. This results in a refinement into at most $p(|J|)^{N_\ell}$ many sets.
    
    The second option is partitioning $P_G$ according to the joint sign patterns of the $g_{i,k}\circ x^{(\ell-1)}_{J}$, where $i\in[N_\ell]$, $k\in[h_i]$. Specifically, for each $\sigma\in\prod_{i\in[N_\ell]}\{-1,0,1\}^{h_i}$, we define
    \begin{align*}
        R_\sigma:=\{z\in P_G\colon\operatorname{sgn}g_{i,k}(x^{(\ell-1)}_{J}(z))=\sigma_{i,k}\ \forall i\in[N_\ell],k\in[h_i]\} \,,
    \end{align*}
    and note that for nonempty $R_\sigma$ and each $i\in [N_\ell]$, Assumption~\ref{ass:model}(ii) provides one $A_i\in \mathcal{A}_i$ such that $x^{(\ell-1)}_{J}(R_\sigma)\subseteq A_i$.
    Using \Cref{lem:sign_counting} and the normalisation argument above, we can see that the number of nonempty $R_\sigma$ is bounded by
    \begin{align*}
        \sum_{j=0}^{d_{\ell-1}}2^j{N_\ell h(|J|)\choose j}\,,
    \end{align*}
    where we used that, for every $y\in Y_P\cap Q_G$, its normalisation $\widehat{y}$ is in $\widehat{Y}_P\cap Q_G$ and $\operatorname{sgn}g_{i,k}(\widehat{y})=\operatorname{sgn}g_{i,k}(y)$ due to Assumption~\ref{ass:model}(iii).
    By choosing the refinement option with the smaller bound for each $G$, we obtain a subdivision of $P$ into at most
    \begin{align}\label{eq:layer_est}
        \sum_{j=0}^{d_{\ell-1}}2^j{{|J|\choose 2}\choose j}\min\left\{p(|J|)^{N_\ell},\sum_{j=0}^{d_{\ell-1}}2^j{N_\ell h(|J|)\choose j}\right\} \,,
    \end{align}
    many sets.\\

    It remains to verify that every element $P'$ of the newly obtained partition satisfies the induction properties.
    First, every $P'$ is polyhedral as it is derived by intersecting a polyhedral set with finitely many affine equality or inequality conditions, where we note that affine (in)equalities on $x^{(\ell-1)}_{J}(z)$ induce affine (in)equalities on $z$ since $x^{(\ell-1)}_{J}$ is affine on $P$. 
    We next verify \eqref{ind_ass_a}. By construction, for every $i\in[N_\ell]$ there exists $A_i\in\mathcal A_i$ such that $x^{(\ell-1)}_J(P')\subseteq A_i$. 
    For every  $i\in[N_\ell]$ and $z\in P'$, we have
    \begin{align*}
        \cC_{i,\diamond}^{(\ell)}(\rho^{-1}(z))=\cC_{i,G,\diamond}^{(\ell)}(x_J^{(\ell-1)}(z))
    \end{align*}
    as a direct consequence of \Cref{def:trans_spike_time} and \Cref{def:causal_maps}. The right-hand side is constant on $P'$, since $x_J^{(\ell-1)}(P')\subseteq A_i$ and \Cref{ass:model}(i) requires $\cC_{i,G,\diamond}^{(\ell)}$ to be constant on $A_i$. The causal sets in earlier layers remain constant because $P'\subseteq P$. Hence \eqref{ind_ass_a} holds.
    
    Next, we look at \eqref{ind_ass_b}. 
    By \Cref{ass:model}(i), each $f_{i,G}^{(\ell)}|_{A_i}$ is either identically infinite or finite throughout $A_i$. Since $x_J^{(\ell-1)}(P')\subseteq A_i$, the index set $J'$ of finite outputs is constant on $P'$.
    This covers the first part of \eqref{ind_ass_b}.
    For the second part, we note that, in case of $J'\neq\emptyset$, \cref{ass:model}(i) combined with $x^{(\ell-1)}_{J}(P')\subseteq A_i$ ensure that $x^{(\ell)}_{J'}(z)$ can be written as an affine function of $x^{(\ell-1)}_{J}$(z), and thus
    \begin{align*}
        \dim(x^{(\ell)}_{J'}(P'))\leq \min\{\dim(x^{(\ell-1)}_{J}(P')),|J'|\}\leq \min\{d_{\ell-1}+1,N_\ell\}=d_{\ell}+1 \,.
    \end{align*}
    
    In order to verify \eqref{ind_ass_c}, we note that the conditions defining $P'$ are invariant under the respective transformations.
    The respective equivariance properties are guaranteed by \cref{ass:model}(iii).
    Moreover, the respective invariance properties follow from \cref{ass:model}(iii), since \cref{ass:model}(ii) ensures that two points of $Q_G$ that produce the same sign for every $g_{i,k}$ belong to the same element in $\mathcal{A}_i$. 

    Due to monotonicity, $p(|J|)\leq p(N_{\ell-1})$ and $h(|J|)\leq h(N_{\ell-1})$.
    Thus, multiplying \eqref{eq:layer_est} across the layers, we obtain
    \begin{align*}
        |\mathcal{P}_L|\leq \prod_{\ell\in[L]} \sum_{j=0}^{d_{\ell-1}} 2^j{{N_{\ell-1}\choose 2}\choose j}\min\left\{p(N_{\ell-1})^{N_\ell},\sum_{j=0}^{d_{\ell-1}}2^j{N_\ell h(N_{\ell-1})\choose j}\right\} \,.
    \end{align*}
    To see that this transfers to a bound on the number of causal pieces, we define, for $P\in\mathcal{P}_L$ the set
    \begin{align*}
        B_P:=\{b\in\R^{N_0}\colon \rho(b)\in P\} \,,
    \end{align*}
    where $\rho$ acts componentwise. As $\rho$ is bijective the sets $B_P$, $P\in\mathcal{P}_L$, constitute a partition of the original input domain. By \eqref{ind_ass_a}, we know that every neuron's causal set is constant on $B_P$, and thus all elements in $B_P$ have the same causal subnetwork. In particular each $B_P$ is contained in some causal piece, and different causal pieces cannot contain the same $B_P$. 
    Thus there cannot be more causal pieces than elements of $\mathcal{P}_L$, and we have $\mathfrak{C}_\diamond(\Phi)\leq|\mathcal{P}_L|$.
\end{proof}

\begin{highlightbox}
\begin{corollary}\label{cor:hyper_general}
    Let $\diamond\in\{<,\leq\}$, and let $\Phi$ be a feedforward single-spike network with depth $L$ and widths $N_0,\dots,N_L$. Let it satisfy Assumption~\ref{ass:model} with $p\colon\N\to\N$, $h\colon\N\to\N_0$, and transformed spike-time maps $f^{(\ell)}_i$, $\ell\in[L]$, $i\in[N_\ell]$. Suppose, in addition, that for every $\ell\in[L]$, $i\in[N_\ell]$, nonempty $J\subseteq[N_{\ell-1}]$, $q\in[|J|]$, and weak order $G=(G_1,\dots,G_q)$ of J, we have:
    \begin{enumerate}[(i)]
        \item \label{ass:selection_hypothesis} for every $y\in Q_G$, that
        \begin{align*}
            f_{i,G}^{(\ell)}(y)\in\{y_j\colon j\in J\}\cup\{\infty\} 
        \end{align*}
        \item the partition and affine test functions from Assumption~\ref{ass:model} also satisfy $|\mathcal{A}|\leq p(q)$ and $H\leq h(q)$.
    \end{enumerate}
    Then 
    \begin{align*}
        \mathfrak{C}_\diamond(\Phi) \leq (\sum_{q\in[N_0]}q!{N_0\brace q}) \prod_{\ell=1}^L \min\left\{p(d_{\ell-1}+1)^{N_\ell},\ \sum_{j=0}^{d_{\ell-1}}2^j{N_\ell h(d_{\ell-1}+1)\choose j}\right\},
    \end{align*}
    where $d_\ell:=\min_{0\leq r\leq\ell} (N_r-1)$, $\ell\in\{0,\dots,L\}$, and ${N_0\brace q}$ denotes Stirling numbers of the second kind.
\end{corollary}
\end{highlightbox}
\begin{proof}
    Fix $\ell\in[L]$, $i\in[N_\ell]$, a nonempty $J\subseteq[N_{\ell-1}]$, and a weak order $G=(G_1,\ldots,G_q)$ of $J$. Let $\mathcal A$ be the partition provided by Assumption 5.5, and let $A\in\mathcal A$ be a set on which $f_{i,G}^{(\ell)}$ is finite. Put $K_0:=\varnothing$ and $K_k:=G_1\cup\cdots\cup G_k$ for $k\in[q]$. By condition (i) of the corollary, for every $y\in A$, the output equals the common coordinate of some block $G_k$, and hence
    \begin{align*}
    \mathcal C_{i,G,\diamond}^{(\ell)}(y)
    =
    \begin{cases}
    K_{k-1},&\diamond={<},\\
    K_k,&\diamond={\le}.
    \end{cases}
    \end{align*}
    For either fixed convention, these sets are distinct as $k$ varies. Since the causal-set map is constant on $A$ by \Cref{ass:model}(i), the selected index $k$ is independent of $y\in A$. Choosing any $j\in G_k$, we obtain, for every $y\in A$, that
    \begin{align}\label{eq:fycor}
        f_{i,G}^{(\ell)}(y)=y_j.
    \end{align}

    We now construct partitions $\mathcal{P}_0,\dots,\mathcal{P}_L$ of the transformed domain $D^{N_0}$. We retain properties \eqref{ind_ass_a}-\eqref{ind_ass_c} from the proof of \Cref{thm:hyper_general}, and additionally require that, for every $r\in\{0,\dots,L\}$, $P\in \mathcal{P}_r$, and coordinate $i\in[N_r]$ for which $x_i^{(r)}$ is finite on $P$, there exists $j\in[N_0]$ such that, for every $z\in P$, 
    \begin{align}\label{eq:ind_ass_d}
        x_i^{(r)}(z)=z_j.
    \end{align}
    We start with
    \begin{align*}
        \mathcal{P}_0:=\{Q_G\colon G \text{ is a weak order of }[N_0]\}.
    \end{align*}
    Properties \eqref{ind_ass_a}-\eqref{ind_ass_c} and \eqref{eq:ind_ass_d} hold since $x^{(0)}$ is the identity and each $Q_G$ is polyhedral and invariant under translation or positive scaling respectively. There are ${N_0 \brace k}$ partitions of $[N_0]$ into $k$ nonempty sets, each of which admits $k!$ different orderings. Hence
    \begin{align*}
        |\mathcal{P}_0|=\sum_{k\in[N_0]} k! {N_0 \brace k}.
    \end{align*}
    For the induction step, let $\ell\in[L]$ and $P\in\mathcal{P}_{\ell-1}$. Denote by $J\subseteq[N_{\ell-1}]$ the set of coordinates of $x^{(\ell-1)}$ that are finite on $P$. If $J=\emptyset$, the causality assumption \eqref{eq:nfbfi} implies that all outputs in layer $\ell$ are infinite and we can retain $P$ without subdivision.

    Suppose $J\neq\emptyset$. By \eqref{eq:ind_ass_d}, each coordinate of $x_J^{(\ell-1)}$ equals a fixed transformed input coordinate on $P$. 
    Since the partitions are constructed by iterative refinement of the starting partition, $P$ lies in one set of $\mathcal{P}_0$, where the weak order of these input coordinates is fixed. It follows that $x_J^{(\ell-1)}$ has one weak order $G=(G_1,\dots,G_q)$ throughout $P$.
    Hence no subdivision according to weak orders is needed here.

    Moreover, every finite output equals an incoming coordinate by assumption \eqref{ass:selection_hypothesis}. Therefore, for each input $z$, the set of finite values in each layer is contained in the set of finite values in the preceding layer. The $q$ distinct values entering layer $\ell$ must consequently occur in every preceding layer, which implies
    \begin{align*}
        q\leq \min_{0\leq r\leq\ell-1}N_r=d_{\ell-1} + 1.
    \end{align*}

    For each neuron $i\in[N_\ell]$, choose a partition $\mathcal{A}_i$ of $Q_G$ into at most $p(q)$ sets and affine functions $g_{i,1},\dots, g_{i,H_i}$ with $H_i\leq h(q)$, satisfying Assumption~\ref{ass:model}(i)-(iii). 

    As in the proof of \Cref{thm:hyper_general}, we can refine $P$ in either of two ways. Taking, for all $(A_1,\dots,A_{N_\ell})\in\prod_{i\in[N_\ell]}\mathcal{A}_i$, the nonempty sets
    \begin{align*}
        \{z\in P\colon x_J^{(\ell-1)}(z)\in A_i\ \forall i\in[N_\ell]\}
    \end{align*}
    yields at most $p(q)^{N_\ell}$ many sets. Alternatively, we can partition $P$ according to the signs of all
    \begin{align*}
        g_{i,k}\circ x_J^{(\ell-1)},
    \end{align*}
    with $i\in[N_\ell]$, $k\in[H_i]$.
    Due to properties \eqref{ind_ass_b}-\eqref{ind_ass_c}, the normalisation argument in the proof of \Cref{thm:hyper_general} ensures that the same joint sign patterns are realised on a set of affine dimension at most $d_{\ell-1}$. \Cref{lem:sign_counting} therefore bounds the number of sets of this refinement by
    \begin{align*}
        \sum_{j=0}^{d_{\ell-1}}2^j{N_\ell h(q)\choose j}.
    \end{align*}
    For either refinement, every resulting set $P'$ satisfies
    \begin{align*}
        x_J^{(\ell-1)}(P')\subseteq A_i
    \end{align*}
    for some $A_i\in\mathcal{A}_i$, $i\in[N_\ell]$.
    This follows directly from the first construction and from Assumption~\ref{ass:model}(ii) for the second option. The respective argument  
    in the proof of \Cref{thm:hyper_general} therefore preserves properties \eqref{ind_ass_a}-\eqref{ind_ass_c}. By \eqref{eq:fycor}, each finite output on $P'$ also equals one fixed incoming coordinate. Applying \eqref{eq:ind_ass_d} shows that this coordinate must equal a fixed transformed input coordinate $z_j$. Thus \eqref{eq:ind_ass_d} is preserved by the induction step.

    Choosing the refinement option with the smaller bound and using $q\leq d_{\ell-1}+1$ as well as monotonicity of $p$ and $h$, we subdivide 
    $P$ into at most
    \begin{align*}
        \min\{p(d_{\ell-1}+1)^{N_\ell}, \sum_{j=0}^{d_{\ell-1}}2^j{N_\ell h(d_{\ell-1}+1)\choose j}\}
    \end{align*}
    sets. Multiplying these factors over the layers and by $|\mathcal{P}_0|$ yields the claimed bound, where we note that property \eqref{ind_ass_a} ensures that the preimage $\rho^{-1}(P)$ of every set $P\in\mathcal{P}_L$ is contained in one causal piece.  
\end{proof}

\section{Hyperplane bounds for specific neuron models satisfying \cref{ass:model}}

In this section we consider networks without delays, i.e.\@ each layer simply receives the spike times of the previous layers as input.

\subsection{Upper bound for nLIF with exponential synapses}

\begin{highlightbox}
\begin{theorem}\label{thm:nLIF_hyperplane}
    Let $\Phi$ be a single-spike neural network using nLIF neurons with exponential synapses according to \Cref{def:nLIF}, with depth $L$, layer dimensions $N_0,\dots,N_L$, and common time constant. 
    Then
    \begin{align*}
        \mathfrak{C}_<(\Phi) \leq \prod_{\ell\in[L]} \sum_{j=0}^{d_{\ell-1}} 2^j{{N_{\ell-1}\choose 2}\choose j}\min\{N_{\ell-1}^{N_\ell},\sum_{j=0}^{d_{\ell-1}}2^j{N_\ell (N_{\ell-1}-1)\choose j}\},
    \end{align*}
    where $d_\ell:=\min_{0\leq r\leq\ell} (N_r-1)$, $\ell\in\{0,\dots,L\}$. 
\end{theorem}
\end{highlightbox}
\begin{proof}
    We verify Assumption~\ref{ass:model} with $\rho(t)=e^{\frac{t}{\tau}}$, $p(n)=n$, and $h(n)=n-1$.
    For every $\ell\in[L]$ and $i\in[N_\ell]$, choose $f_i^{(\ell)}$ by applying the first-threshold-hit rule of \Cref{def:nLIF} to the incoming spike-time vector and expressing the resulting output in transformed coordinates. These maps satisfy \eqref{eq:nfbfi}.

    Fix $\ell\in[L]$, $i\in[N_\ell]$, a nonempty set $J\subseteq[N_{\ell-1}]$ and a weak order $G=(G_1,\dots,G_q)$ of $J$.
    We construct a partition $\mathcal{A}$ of $Q_G$ and affine functions $g_1,\dots,g_H\colon\R^J\to \R$, for which we then verify properties (i)-(iii). To this end, we choose, for every $k\in[q]$ one $j_k\in G_k$ and define the projection
    \begin{align*}
        \eta_k\colon\R^J&\to\R,\\
        y&\mapsto y_{j_k}.
    \end{align*}
    We set $\eta_{q+1}(y):=\infty$ and note that on $Q_G$ we have $0<\eta_1(y)<\dots<\eta_{q+1}(y)$ by definition of a weak order.
    In addition, we define, for $k\in[q]$,
    \begin{align*}
        K_k:=\bigcup_{r\in[k]} G_r,\quad W_k:=\sum_{j\in K_k}W_{ij}^{(\ell)}.
    \end{align*}
    We have, for every $y\in Q_G$, $k\in[q]$, and $z\in[\eta_k(y),\eta_{k+1}(y))$, that
    \begin{align}\label{eq:trans_potential}
        z(u_i^{(\ell)}(\tau\log z)-\vartheta)=(W_k-\vartheta)z - \sum_{j\in K_k}W_{ij}^{(\ell)}y_j. 
    \end{align}
    Note that $t = \tau\log z$ is the argument of $u_i^{(\ell)}$ here. For a given $y\in Q_G$ the corresponding incoming spike times in untransformed coordinates are $t_j^{(\ell-1)}=\tau\log y_j$, $j\in J$.
    At time $t=\tau\log z$ with $z\in[\eta_k(y),\eta_{k+1}(y))$ we have
    \begin{align*}
        e^{-\frac{t-t_j^{(\ell-1)}}{\tau}}=\frac{y_j}{z}.
    \end{align*}
    Thus, in transformed coordinates the condition of reaching the threshold $\vartheta$ is given, for $z\in[\eta_k(y),\eta_{k+1}(y))$, by the affine equation
    \begin{align*}
        (W_k-\vartheta)z-\sum_{j\in K_k}W_{ij}^{(\ell)}y_j=0.
    \end{align*}
    We define, for $k\in[q]$, the linear functions
    \begin{align*}
        g_k\colon\R^J&\to\R\\
        y&\mapsto (W_k-\vartheta)\eta_k(y)-\sum_{j\in K_k}W_{ij}^{(\ell)}y_j.
    \end{align*}
    Moreover, for $k\in[q-1]$, we set
    \begin{align*}
        A_k:=\{y\in Q_G\colon g_r(y)<0 \ \forall r\in[k], g_{k+1}(y)\geq0 \}
    \end{align*}
    as well as $A_q:=\{y\in Q_G\colon g_r(y)<0 \ \forall r\in[q]\}$. Since $g_1(y)=-\vartheta\eta_1(y)<0$ the nonempty sets amongst $A_1,\dots,A_q$ form a polyhedral partition $\mathcal{A}$ of $Q_G$. 
    
    For $k<q$, equation \eqref{eq:trans_potential} also holds at $z=\eta_{k+1}(y)$, because the responses of spikes arriving at that time are zero. Its values at the two endpoints are $g_k(y)$ and $g_{k+1}(y)$. The potential is zero before the first arrival, and negative endpoint values exclude a threshold hit between consecutive arrivals, since the expression in \eqref{eq:trans_potential} is affine in $z$. Consequently, for every $k<q$ and $y\in A_k$,
    \begin{align*}
    \eta_k(y)<f_{i,G}^{(\ell)}(y)\le\eta_{k+1}(y).
    \end{align*}
    The exclusive causal set is therefore $K_k$. Moreover, if $A_k\ne\varnothing$, the endpoint inequalities force the slope in \eqref{eq:trans_potential} to be positive, so $W_k>\vartheta$.

    On $A_q$, the potential remains below threshold through the last arrival. Equation \eqref{eq:trans_potential} therefore implies that the output is finite precisely when $W_q>\vartheta$. In that case its exclusive causal set is $K_q$; otherwise, the output is infinite throughout $A_q$ and its causal set is empty.
    Thus, on every nonempty $A_k$ with finite output, we have $W_k>\vartheta$ and, for every $y\in A_k$, that
    \begin{align*}
        f_{i,G}^{(\ell)}(y)
        =
        \frac{\sum_{j\in K_k}W_{ij}^{(\ell)}y_j}
             {W_k-\vartheta}.
    \end{align*}
    Every finite restriction is therefore linear, and finite output status and the exclusive causal set are constant on each set in $\mathcal A$. Since
    \begin{align*}
        |\mathcal A|\le q\le |J|,
    \end{align*}
    \Cref{ass:model}(i) holds with $p(n)=n$.
    For \Cref{ass:model}(ii), the signs of $g_2,\ldots,g_q$ determine membership in $\mathcal A$, since $g_1$ is always negative on $Q_G$. Hence $q-1\le |J|-1$ test functions suffice, giving $h(n)=n-1$.
        
    In order to verify Assumption~\ref{ass:model}(iii)(H), we fix $\lambda\in(0,\infty)$. Note that, for every $y\in Q_G$ we have $\lambda y\in Q_G$ and $g_r(\lambda y)=\lambda g_r(y)$ due to linearity of $g_r$, $r\in[q]$. As their signs determine the partition we also have $y\in A_k\implies \lambda y\in A_k$. 
    On every nonempty $A_k$ with finite output, the displayed formula is linear in $y$. Since positive scaling preserves $A_k$, it follows, for every $y\in A_k$, that
    \begin{align*}
    f_{i,G}^{(\ell)}(\lambda y)
    =
    \lambda f_{i,G}^{(\ell)}(y).
    \end{align*}

    Having ensured that Assumption~\ref{ass:model} holds for this model, we can apply \Cref{thm:hyper_general} to obtain the claimed bound.
\end{proof}

\subsection{Upper bound for ReLU-type response model}

\begin{highlightbox}
\begin{theorem}\label{thm:ReLU_hyperplane}
    Let $\Phi$ be a single-spike neural network using ReLU-type response neurons according to \Cref{def:relu_response} with depth $L$ and layer dimensions $N_0,\dots,N_L$. 
    Then
    \begin{align*}
        \mathfrak{C}_{<}(\Phi) \leq \prod_{\ell\in[L]} \sum_{j=0}^{d_{\ell-1}} 2^j{{N_{\ell-1}\choose 2}\choose j}\min\{N_{\ell-1}^{N_\ell},\sum_{j=0}^{d_{\ell-1}}2^j{N_\ell (N_{\ell-1}-1)\choose j}\},
    \end{align*}
    where $d_\ell:=\min_{0\leq r\leq\ell} (N_r-1)$, $\ell\in\{0,\dots,L\}$. 
\end{theorem}
\end{highlightbox}
\begin{proof}
    We verify Assumption~\ref{ass:model} with $\rho(t)=t$, $p(n)=n$, and $h(n)=n-1$. For $\ell\in[L]$, $i\in[N_\ell]$, $y\in(\R\cup\{\infty\})^{N_{\ell-1}}$, we define
    \begin{align*}
        f_i^{(\ell)}(y):=\inf\{t\in\R\colon \sum_{j\in[N_{\ell-1}], y_j<\infty} W_{ij}^{(\ell)}\max\{0,t-y_j\}=\vartheta\},
    \end{align*}
    where we use the convention $\inf\emptyset = \infty$. The potential is zero before the earliest finite incoming spike, and identically zero if all incoming times are infinite. Thus the causality condition \eqref{eq:nfbfi} holds.
    Fix $\ell\in[L]$, $i\in[N_\ell]$, a nonempty set $J\subseteq[N_{\ell-1}]$ and a weak order $G=(G_1,\dots,G_q)$ of $J$.
    We construct a partition $\mathcal{A}$ of $Q_G$ and affine functions $g_1,\dots,g_H\colon\R^J\to \R$, for which we then verify properties (i)-(iii). To this end, we choose, for every $k\in[q]$ one $j_k\in G_k$ and define the projection
    \begin{align*}
        \eta_k\colon\R^J&\to\R,\\
        y&\mapsto y_{j_k}.
    \end{align*}
    We set $\eta_{q+1}(y):=\infty$. In addition, we define, for $k\in[q]$,
    \begin{align*}
        K_k:=\bigcup_{r\in[k]} G_r,\quad W_k:=\sum_{j\in K_k}W_{ij}^{(\ell)}.
    \end{align*}
    Note that for $k\in[q]$, $y\in Q_G$, $t\in[\eta_k(y),\eta_{k+1}(y))$, we have
    \begin{align}\label{eq:ramp_aff_eq}
        \sum_{j\in J}W_{ij}^{(\ell)}\max\{0,t-y_j\}-\vartheta=W_k t-\sum_{j\in K_k}W_{ij}^{(\ell)}y_j-\vartheta, 
    \end{align}
    which is an affine function of $t$. Consequently we define for $k\in[q]$, the affine function
    \begin{align}\begin{split}\label{eq:ramp_aff_fun}
        g_k\colon \R^J&\to\R\\
        y&\mapsto W_k\eta_k(y)-\sum_{j\in K_k} W_{ij}^{(\ell)}y_j-\vartheta.
    \end{split}\end{align}
    On $Q_G$, its value is the potential minus threshold at the $k$-th arrival.
    We define, for every $k\in[q-1]$, the sets
    \begin{align*}
    A_k:=
    \{y\in Q_G:g_r(y)<0\ \forall r\in[k],
                      \ g_{k+1}(y)\ge0\}
    \end{align*}
    and
    \begin{align*}
    A_q:=
    \{y\in Q_G:g_r(y)<0\ \forall r\in[q]\}.
    \end{align*}
    Since $g_1(y)=-\vartheta<0$, the nonempty sets among $A_1,\ldots,A_q$ form a polyhedral partition $\mathcal A$ of $Q_G$.
    For $k<q$, equation \eqref{eq:ramp_aff_eq} also holds at $t=\eta_{k+1}(y)$, because newly arriving responses vanish there. 
    The argument in the proof of \Cref{thm:nLIF_hyperplane} that locates the first threshold hit from the signs of $g_1(y),\dots,g_{k+1}(y)$ applies equally to the affine expression in \eqref{eq:ramp_aff_eq}. Thus, for every $y\in A_k$,
    \begin{align*}    
        \eta_k(y)<f_{i,G}^{(\ell)}(y)\le\eta_{k+1}(y).
    \end{align*} 
    Hence the exclusive causal set is $K_k$, and $A_k\ne\varnothing$ implies $W_k>0$.
    On $A_q$, equation \eqref{eq:ramp_aff_eq} shows that the output is finite precisely when $W_q>0$. Its exclusive causal set is then $K_q$. Otherwise, the output is infinite throughout $A_q$ and its causal set is empty. On every nonempty $A_k$ with finite output, we consequently have, for every $y\in A_k$, that
    \begin{align*}
    f_{i,G}^{(\ell)}(y)
    =
    \frac{\vartheta+\sum_{j\in K_k}W_{ij}^{(\ell)}y_j}
         {W_k}.
    \end{align*}
    Thus every finite restriction is affine, and finite output status and the exclusive causal set are constant on each set in $\mathcal A$. Since $|\mathcal A|\le q\le |J|$, \Cref{ass:model}(i) holds with $p(n)=n$.
    For \Cref{ass:model}(ii), the signs of $g_2,\ldots,g_q$ determine membership in $\mathcal A$, since $g_1$ is always negative. Hence $h(n)=n-1$ suffices.
    
    To verify Assumption~\ref{ass:model}(iii)(T), we fix $c\in\R$. Due to \eqref{eq:ramp_aff_fun}, we have $g_k(y+c\mathbf{1})=g_k(y)$ for all $k\in[q]$, $y\in\R^J$. Moreover, translation by $c\mathbf{1}_J$ preserves $Q_G$, and since test signs determine membership in $\mathcal{A}$, it also preserves each set in $\mathcal{A}$. Thus we have $f_{i,G}^{(\ell)}(y+c\mathbf{1}_J)=f_{i,G}^{(\ell)}(y)+c$ for every $y\in Q_G$ with finite output. 
    Having verified Assumption~\ref{ass:model}, application of \Cref{thm:hyper_general} yields the claimed bound. 
\end{proof}

\begin{remark}
    For inclusive causal pieces, the bounds in \Cref{thm:nLIF_hyperplane} and \Cref{thm:ReLU_hyperplane} hold with $N_{\ell-1}^{N_\ell}$ replaced by $(N_{\ell-1}+1)^{N_\ell}$. Specifically, for each fixed weak order, partition according to $\eta_k\le f_{i,G}^{(\ell)}<\eta_{k+1}$, $k\in[q]$, and $f_{i,G}^{(\ell)}=\infty$, where $\eta_{q+1}:=\infty$. The same test functions determine these polyhedral sets. On a finite-output set, the corresponding affine formula applies when its denominator is positive. Otherwise, a first threshold hit can occur only at the left endpoint, implying $f_{i,G}^{(\ell)}=\eta_k$. Thus \Cref{ass:model} holds with $p(n)=n+1$ and $h(n)=n-1$.
\end{remark}

\subsection{Upper bound LIF with Dirac-delta synapses}
\begin{highlightbox}
\begin{theorem}\label{thm:LIF}
    Let $\diamond\in\{<,\leq\}$ and let $\Phi$ be a single-spike neural network using LIF neurons with Dirac-$\delta$ synapses according to \Cref{def:dirac} with depth $L$, layer dimensions $N_0,\dots,N_L$, and common time constant. 
    Then
    \begin{align*}
        \mathfrak{C}_{\diamond}(\Phi) \leq (\sum_{q\in[N_0]}q!{N_0\brace q}) \prod_{\ell=1}^L \min\left\{(d_{\ell-1}+1)^{N_\ell},\ \sum_{j=0}^{d_{\ell-1}}2^j{N_\ell d_{\ell-1}\choose j}\right\},
    \end{align*}
    where $d_\ell:=\min_{0\leq r\leq\ell} (N_r-1)$, $\ell\in\{0,\dots,L\}$, and ${N_0\brace q}$ denotes Stirling numbers of the second kind.
\end{theorem}
\end{highlightbox}
\begin{proof}
    We verify Assumption~\ref{ass:model} and the additional requirements of \Cref{cor:hyper_general} with $\tau\in(0,\infty)$, $\rho(t)=e^{\frac{t}{\tau}}$, $D=(0,\infty)$, $p(n)=n$, $h(n)=n-1$. 
    Fix $\ell\in[L]$ and $i\in[N_\ell]$. The potential is zero before the earliest incoming spike. Between arrivals, a potential below the positive threshold remains below it, since it decays towards zero. Thus every finite output spike occurs at an incoming spike time.
    For $y\in(D\cup\{\infty\})^{N_{\ell-1}}$, we have incoming times $t_j^{(\ell-1)}=\rho^{-1}(y_j)$. At finite incoming time $\tau\log y_k$, the potential is
    \begin{align*}
        \tfrac{1}{y_k}\sum_{j\in[N_{\ell-1}],y_j\leq y_k} W_{ij}^{(\ell)}y_j.
    \end{align*}
    The sum includes all simultaneous incoming spikes, as prescribed by \Cref{def:dirac}. The transformed output is therefore
    \begin{align}\label{eq:LIF_fi}
        f_i^{(\ell)}(y):=\min\{y_k\colon k\in[N_{\ell-1}], y_k<\infty, \sum_{j\in[N_{\ell-1}], y_j\leq y_k}W_{ij}^{(\ell)}y_j\geq\vartheta y_k\},
    \end{align}
    where we use the convention $\min\emptyset:=\infty$. Consequently $f_i^{(\ell)}(y)$ satisfies the causality condition \eqref{eq:nfbfi}. It moreover establishes the requirement that every finite output equals an incoming coordinate. 

    It remains to construct the partitions and test functions demanded by Assumption~\ref{ass:model}. Fix $J\subseteq[N_{\ell-1}]$, $q\in[|J|]$, and a weak order $G=(G_1,\dots,G_q)$ of $J$. For each $k\in[q]$, choose $j_k\in G_k$ and define $\eta_k\colon \R^J\to\R, y\mapsto y_{j_k}$, and $K_k:=\bigcup_{r\in[k]}G_r$. We also set $K_0:=\emptyset$. On $Q_G$ the values $\eta_1(y),\dots,\eta_q(y)$ are strictly increasing and $K_k$ consists precisely of the indices $j\in J$ with $y_j\leq \eta_k(y)$. Thus we can define, for $k\in[q]$, the test functions
    \begin{align*}
        g_k\colon \R^J&\to\R\\
        y&\mapsto \sum_{j\in K_k} W_{ij}^{(\ell)}y_j-\vartheta\eta_k(y),
    \end{align*}
    and note that, due to \eqref{eq:LIF_fi}, $f_{i,G}^{(\ell)}(y)$ evaluates to the first $\eta_k(y)$ for which $g_k(y)\geq 0$, or $\infty$ in case no such index exists.

    To check Assumption~\ref{ass:model}(i), we partition $Q_G$ according to that first index. Specifically, we define, for $k\in[q]$,
    \begin{align*}
        A_k:=\{y\in Q_G\colon g_r(y)<0\ \forall r<k, g_k(y)\geq 0\}
    \end{align*}
    and $A_\infty:=\{y\in Q_G\colon g_r(y)<0\ \forall r\in[q]\}$.
    The nonempty sets amongst these form a polyhedral partition $\mathcal{A}$ of $Q_G$. We have, for every $k\in[q]$ and $y\in A_k$, 
    \begin{align}\label{eq:LIF_fiC}
        f_{i,G}^{(\ell)}(y) =\eta_k(y), \quad \mathcal{C}_{i,G,\diamond}^{(\ell)}(y)=\begin{cases}
            K_{k-1}, &\diamond = <\\K_k, & \diamond =\leq
        \end{cases}. 
    \end{align}
    On $A_\infty$, the output is infinite and the causal set is empty. Thus every causal set is constant on a given set of the partition, and every finite restriction is linear. 
    To bound the cardinality of $\mathcal{A}$, we observe that, for every $y\in Q_G$, we have
    \begin{align*}
        g_1(y)=(\sum_{j\in G_1}W_{ij}^{(\ell)}-\vartheta)\eta_1(y).
    \end{align*}
    Due to $\eta_1(y)>0$, the sign of $g_1$ is constant on $Q_G$. If it is nonnegative, we have $A_1=Q_G$, and otherwise $A_1=\emptyset$, leaving at most the sets $A_2,\dots, A_q,A_\infty$. Hence $|\mathcal{A}|\leq q=p(q)\leq p(|J|)$, which verifies Assumption~\ref{ass:model}(i) and the additional partition bound requirement of \Cref{cor:hyper_general}.

    For Assumption~\ref{ass:model}(ii), the same observation allows us to omit $g_1$ from the test functions, as its sign is always the same. The signs of $g_2,\dots,g_q$ determine membership in $A_1,\dots,A_q,A_\infty$. These $q-1=h(q)\leq h(|J|)$ linear functions therefore verify Assumption~\ref{ass:model}(ii) and the additional test bound requirement of \Cref{cor:hyper_general}.

    For Assumption~\ref{ass:model}(iii)(H), let $\lambda>0$ and $y\in Q_G$. Then we have $\lambda y\in Q_G$ and $g_k(\lambda y)=\lambda g_k( y)$ for every $k\in[q]$. 
    Thus $y$ and $\lambda y$ belong to the same set in $\mathcal{A}$. If the output is finite, they belong to $A_k$ for some $k\in[q]$, so \eqref{eq:LIF_fiC} yields equivariance of $f_{i,G}^{(\ell)}$. 
    Having verified that the requirements are met, application of \Cref{cor:hyper_general} establishes the claim.
\end{proof}

\begin{remark}
The $2^j$ factor could be removed by taking into account that it would be sufficient to determine partition membership by only distinguishing whether the test functions are negative or nonnegative.   
\end{remark}

\subsection{Upper bound nLIF with Dirac-delta synapses}
\begin{highlightbox}
\begin{theorem}\label{thm:nLIF_Dirac}
    Let $\diamond\in\{<,\leq\}$ and let $\Phi$ be a single-spike neural network using nLIF neurons with Dirac-$\delta$ synapses according to \Cref{def:instant} with depth $L$ and layer dimensions $N_0,\dots,N_L$. 
    Then
    \begin{align*}
        \mathfrak{C}_{\diamond}(\Phi) \leq \sum_{q\in[N_0]}q!{N_0\brace q} 
    \end{align*}
    where ${N_0\brace q}$ denotes Stirling numbers of the second kind.
\end{theorem}
\end{highlightbox}
\begin{proof}
    Fix a weak order $G$ of $[N_0]$. We show inductively that, on $Q_G\subseteq\R^{N_0}$, the indices of finite spike times and their weak order are fixed in every layer, and all causal sets are constant. At the input layer, all spike times are finite and their weak order is $G$.
    For the induction step, fix $\ell\in[L]$. If all incoming spike times are infinite, every neuron in layer $\ell$ has potential zero, infinite spike time, and empty causal set. 
    Otherwise, let $H=(H_1,\dots,H_q)$ be the fixed weak order of the finite incoming spike times. By \Cref{def:instant}, the potential of neuron $i\in[N_\ell]$ at the arrival of block $H_k$, $k\in[q]$, is
    \begin{align*}
        \sum_{s\in[k]}\sum_{j\in H_s} W_{ij}^{(\ell)}.
    \end{align*}
    This value is independent of the input $b\in Q_G$. Since the potential is zero before the first arrival and constant between arrivals, the neuron spikes at the first block for which this value is at least $\vartheta$.
    If no such block exists, its spike time is infinite and its causal set is empty throughout $Q_G$. 
    Otherwise, the first such index $k_i$is independent of $b\in Q_G$, and, for every $b\in Q_G$,
    \begin{align*}
        \mathcal C_{i,\diamond}^{(\ell)}(b)=
        \begin{cases}
        \bigcup_{r\in[k_i-1]}H_r,&\diamond={<},\\[2mm]
        \bigcup_{r\in[k_i]}H_r,&\diamond={\le}.
\end{cases}
    \end{align*}
    Its spike time equals the common incoming spike time of $H_{k_i}$.
    Thus the indices of finite outputs are fixed. Moreover, their weak order is fixed, since outputs selecting the same block coincide, and outputs selecting earlier blocks occur earlier. This completes the induction. 

    Every $Q_G$ is therefore contained in one causal piece. There are ${N_0 \brace k}$ partitions of $[N_0]$ into $k$ nonempty sets, each admitting $k!$ orders. As the corresponding weak-order domains partition the input space, the claim follows.
\end{proof}

%% file: content/acknowledgments.tex
D.D. acknowledges funding through the Horizon Europe's Marie Skłodowska-Curie Actions (MSCA) Project 101103062 (BASE) and the European Space Agency's OSIP Co-funded Research Programme (4000152390). 
D.E. was funded by the Horizon Europe's Marie Skłodowska-Curie Actions (MSCA) Project 101152591 (CompAS).

%% file: content/appendix.tex
\subsection{nLIF neuron with approximately ReLU-type response}\label[app]{app:nlif-lower-bound}

Fix a finite-width and finite-depth SNN with ReLU-type response neurons from the lower-bound construction in the introduction. 
For every causal piece counted in the lower-bound construction, choose one representative input in its interior. 
Consider now an nLIF neural network with exponential synapses of the same architecture (including weights and thresholds).
The response function of nLIF neurons is (\cref{def:nLIF})
\begin{equation*}
g_\tau(s)=(1-e^{-s/\tau})\mathbf 1_{\{s\ge0\}} = (1-e^{-\max(0,s)/\tau})\,.
\end{equation*}
Keep the thresholds unchanged, compress all input spike times by $\varepsilon>0$, and multiply each synaptic weight by $\tau/\varepsilon$. 
In rescaled time, the effective response is
\begin{equation}\label{eq:si_convergence}
k_\varepsilon(s)
=\frac{\tau}{\varepsilon}(1-e^{-\varepsilon \max(0,s)/\tau}),
\qquad
\sup_{|s|\le B}|k_\varepsilon(s)- \max(0,s)|
\le \frac{\varepsilon B^2}{2\tau} \,.
\end{equation}
Thus $k_\varepsilon$ converges uniformly to the ReLU-type response on bounded intervals.
To show this, we use the following identity:
\begin{equation}
    \text{For every $u \geq 0$:} \quad 
u-(1-e^{-u})=\int_0^u(1-e^{-v})\,dv \,.
\end{equation}
Since $0\le1-e^{-v}\le v$ for $v \geq 0$, we find
\begin{equation}
0\le u-(1-e^{-u})\le\int_0^u v\,dv=\frac{u^2}{2} \,.
\end{equation}
Substituting $u=\varepsilon \max(s,0)/\tau$, multiplying the inequality by $\tau/\varepsilon$, and assuming $|s| \leq B$, \cref{eq:si_convergence} follows immediately. The chosen ReLU-response network has nonnegative weights and positive thresholds, and all its neurons have finite spike times. Its potentials and those of the approximating nLIF network are therefore nondecreasing, and each ReLU-response potential crosses its threshold with positive slope.
The rescaled nLIF spike times converge to those of the ReLU-response network layer by layer. 
Convergence of the incoming spike times and the response estimate above imply uniform convergence of the rescaled nLIF potentials on bounded time intervals. 
If the corresponding ReLU-response neuron fires at time $T$, its potential has positive slope there. 
Thus, for any sufficiently small $\rho>0$, its potential is below threshold at $T-\rho$ and above threshold at $T+\rho$.
Both endpoint values have a positive distance from the threshold.
For sufficiently small $\varepsilon$, the approximation error is smaller than these distances, so the rescaled nLIF potential remains below threshold at $T-\rho$ and above threshold at $T+\rho$.
By continuity and monotonicity, the rescaled nLIF spike time therefore lies in $(T-\rho,T+\rho)$, proving convergence to $T$.

At each selected input, the arrival of every relevant presynaptic spike occurs strictly before or strictly after the receiving neuron fires. 
Convergence of the rescaled spike times preserves these orderings for sufficiently small $\varepsilon>0$, and hence preserves the corresponding causal subnetwork. 
Since there are finitely many selected inputs and comparisons, a single positive $\varepsilon$ can be chosen to preserve them all. 
The nLIF network therefore realises at least the same number of distinct causal subnetworks.

\subsection{Neurons with Dirac-$\delta$ synapses}\label[app]{sec:Dirac_lower}
The bound for nLIF networks with Dirac-$\delta$ synapses is attained exactly. Indeed, for any input dimension $n\in\mathbb N$ and depth $L\in\mathbb N$, choose all layer widths equal to $n$ and $W^{(\ell)}=\vartheta I_n$ for $\ell\in[L]$. Then $t^{(\ell)}(b)=b$, so the final-layer causal sets determine the complete weak order of the input. Consequently,
\begin{align*}
    C_\diamond(\Phi)
    =\sum_{q=1}^{n}q!
    \left\{\begin{matrix}n\\q\end{matrix}\right\},
    \qquad \diamond\in\{<,\le\}.
\end{align*}

For LIF networks with Dirac-$\delta$ synapses, the width exponent can be shown to be sharp. 
To do so we employ a type of bit-extraction argument, which has previously been used in neural network approximation theory~\cite{liang2017why} and for spiking neural networks to emulate activation functions~\cite{stockl2021optimized}.
For fixed input dimension $d+1$ and depth $L$, with $d,L\in\mathbb N$, the maximal number of causal pieces among networks with common time constant and subsequent widths at most $N$ is then $\Theta_{d,L}(N^{dL})$. 

To establish the lower bound, fix an integer $m\ge2$ and consider inputs $(0,T_1,\ldots,T_d)$, where $T_a>0$ for $a\in[d]$. 
Put $x_a=e^{-T_a/\tau}$. 
We construct a network where layer $r\in[L]$ extracts the $r$-th base-$m$ digit $D_{a,r}$ of each transformed input $x_a$.
For each \(a\in[d]\), layer $r$ contains $m-1$ new neurons whose spike times are $T_a$ when $D_{a,r}\ge k$, and $\infty$ otherwise, for $k\in[m-1]$. 
For example, with $m=4$ and $x_a=0.6$, two quarters fit into $x_a$ but three do not, so $D_{a,1}=2$.
After subtracting $2/4$ from $x_a$, the remainder is $0.1$, into which one sixteenth fits but two do not,
so $D_{a,2}=1$, i.e., $x_a = 2/m + 1/m^2 + \mathrm{remainder}$.
Reference spikes and previously extracted digits are copied through subsequent layers using dummy neurons and connections with weights equal to the threshold, ensuring identical spike times.

Specifically, the neuron in layer $r$ testing $D_{a,r}\ge k$, for $k\in[m-1]$, has the following incoming connections from the preceding layer:
\begin{itemize}
    \item one connection from the copy of the reference neuron firing at time $0$,
    with weight $\vartheta/2$,
    \item one connection from the copy of the input neuron firing at time $T_a$,
    with weight $\vartheta-\vartheta k/(2m^r)$,
    \item connections for each $s\in[r-1]$, from each of the $m-1$ retained test neurons encoding $D_{a,s}$, with weight
    $-\vartheta/(2m^s)$.
\end{itemize}
All other incoming weights are zero.
Before time $T_a$, this neuron's potential is below the threshold, as only the reference spike at time $0$ arrived so far.
At time $T_a$, the neuron's membrane potential is
\begin{align*}
    u(T_a)= \tfrac{\vartheta}{2}e^{\frac{-T_a}{\tau}} + (\vartheta-\tfrac{\vartheta k}{2m^r})  - \sum_{s\in[r-1]}\tfrac{\vartheta D_{a,s}}{2m^s} 
    =\vartheta+\tfrac{\vartheta}{2}
    \left(
        x_a-\sum_{s=1}^{r-1}\tfrac{D_{a,s}}{m^s}
        -\tfrac{k}{m^r}
    \right).
\end{align*}
By definition of the $r$-th base-$m$ digit,
\[
\frac{D_{a,r}}{m^r}
\le x_a-\sum_{s<r}\frac{D_{a,s}}{m^s}
<\frac{D_{a,r}+1}{m^r}.
\]
Hence $u(T_a)\ge\vartheta$ if and only if $D_{a,r}\ge k$ (otherwise the term in brackets becomes negative).
Since the neuron cannot fire before $T_a$ and its potential
subsequently decays towards zero, it fires at $T_a$ precisely when this inequality holds, and otherwise remains silent.

The $d$ branches operate independently, with all cross-branch weights zero. 
Every prescribed digit array $(b_{a,r})\in\{0,\ldots,m-1\}^{dL}$
is realised by choosing
\[
x_a=\sum_{r=1}^{L}\frac{b_{a,r}}{m^r}+\frac{1}{2m^L},
\qquad T_a=-\tau\log x_a.
\]
The final layer contains the last digit tests and copies of all earlier tests. 
Different digit arrays therefore yield different firing/silence patterns in this layer. 
Every firing test or copy has a nonempty exclusive causal set, whereas non-firing neurons have empty causal sets under our definition. 
Consequently,
\begin{align*}
    C_\diamond(\Phi)\ge m^{dL} \qquad \text{for} \qquad \diamond\in\{<,\le\} \,.
\end{align*}
Layer $r$ has width $N_r=d+1+dr(m-1)$. 
For $N\ge d+1+dL$, choosing
\begin{align*}
m=1+\left\lfloor\frac{N-d-1}{dL}\right\rfloor \,,
\end{align*}
keeps all subsequent widths at most $N$ and gives the lower bound $\Omega_{d,L}(N^{dL})$, matching the upper bound for fixed $d,L$.

For a concrete example, take $d=1$, $m=2$, and $L=2$.
The digit-test neuron in the first layer has potential

\begin{align*}
u_1(T)=\vartheta+\frac{\vartheta}{2}\left(x-\frac12\right),
\end{align*}
while the second has potential
\begin{align*}
u_2(T)=\vartheta+\frac{\vartheta}{2}
\left(x-\frac{D_1}{2}-\frac14\right).
\end{align*}
Consider the two inputs $(0,\tau\log(8/3))$ and
$(0,\tau\log(8/5))$, corresponding to $x=3/8$ and $x=5/8$,
respectively:
\[
\begin{array}{c|c|c|c|c}
x & u_1(T)/\vartheta & D_1 & u_2(T)/\vartheta & D_2\\ \hline
3/8 & 15/16 & 0 & 17/16 & 1\\
5/8 & 17/16 & 1 & 15/16 & 0
\end{array}
\]
Thus the first two binary digits are $01$ and $10$.